\documentclass{bmvc2k}

\title{Physics-Informed Modeling for Wood Thermal Analysis and Prediction}

\addauthor{Jingren Xie}{jinxi@dtu.dk}{1, 2}
\addauthor{Alex John Buckthal}{albu@kglakademi.dk}{3}
\addauthor{Ryan Anthony O'Connor}{roc@kglakademi.dk}{3}
\addauthor{Isak Worre Foged}{hfog@kglakademi.dk}{3}
\addauthor{Dim P. Papadopoulos}{dimp@dtu.dk}{1,2}

\addinstitution{
 Technical University of Denmark\\
 Denmark
}
\addinstitution{
 Pioneer Centre for AI\\
 Denmark
}
\addinstitution{
 The Royal Danish Academy\\
 Denmark
}
\runninghead{Xie et al.}{Physics-Informed Modeling}

\usepackage{booktabs}
\usepackage{multirow}
\usepackage{amsfonts}
\usepackage{algorithm}
\usepackage{algpseudocode}
\usepackage{amsmath}
\usepackage{amsthm}
\newtheorem{theorem}{Theorem}[section]

\begin{document}

\maketitle

\begin{abstract}
Wood materials exhibit complex, spatially varying thermal properties that challenge traditional architectural assumptions of material homogeneity. Although data-driven approaches can directly map wood RGB images to their corresponding thermal responses, they operate as uninterpretable black boxes that prioritize statistical correlation and may absorb experimental noise rather than thermodynamic plausibility. To address these limitations, we present physics-informed deep learning frameworks that integrate partial differential equations (PDEs) to predict pixel-level thermal responses of spatially heterogeneous wood materials using wood RGB images and testbed temperature maps. Specifically, we investigate two distinct approaches to enforcing a normalized 2D steady-state heat transfer equation derived from the general heat transfer equation: Physics-Informed Convolutional Neural Networks (PICNNs), which embed physics as a soft penalty term in the loss function, and Physics-Integrated Convolutional Neural Networks (PInteCNNs), which hard-code an analytical approximator-predictor-corrector solver directly into convolutional neural networks. To validate our proposed approaches, we collect three real-world multimodal datasets of Poplar, Grandis Cross-Cut (Grandis-CC), and Grandis Radial-Cut (Grandis-RC) wood samples. We further demonstrate that embedding physical inductive biases successfully balances predictive accuracy,  physical interpretability, and intra-species diversity, outperforming data-driven approaches in handling complex wood material heterogeneity and enabling the extraction of interpretable physical parameters. 
\textit{Project}: {\color{cyan}https://zekifayes.github.io/pim}.
\end{abstract}

\section{Introduction}
\label{sec:intro}
For the past century, wood architectural design has been driven by an ideal of material simplification, prioritizing homogeneity, repeatability, and isothermal conditions. 
This has resulted in a severe misalignment 
with wood materials.
In reality, wood exhibits complex thermal properties, shaped by spatially varying density, leading to diverse thermal responses under given environmental conditions.
A comprehensive understanding of wood thermal responses plays an essential role in quality control and adaptive assembly strategies~\cite{menges2015performative, tibbits2016self, cheng2020multifunctional, fragkia2020methods, fragkia2020wood}. 
Despite their potential, these strategies remain largely conceptual and impractical due to limited foundational knowledge and oversimplified, homogeneous assumptions about wood properties.
While current research has attempted to bridge this gap by exploring propositional design processes on 2D surfaces with mono-modal properties~\cite{fragkia2020methods, fragkia2021predictive, fragkia2023thermodynamic}, a robust predictive model capable of capturing complex wood thermal responses remains lacking.
To overcome these limitations, this study focuses on analyzing, understanding, and predicting the complex thermal responses of wood samples, therefore establishing a foundational predictive framework necessary to harness wood's natural thermal variability for responsive wood design.

To study these thermal behaviors, vision-based methods offer highly effective noncontact and nondestructive approaches for wood material analysis.
Our observation is that wood RGB images and their corresponding thermal responses exhibit highly discernible morphological similarities (Appendix~\ref{sec:wood_analysis_app}), especially when wood has rich grain patterns, suggesting the potential for cross-modal joint feature analysis.
Because wood's physical structure governs how heat flows, wood RGB images can serve as a highly effective visual proxy for spatially varying thermal properties.
We can extract multi-scale visual features to predict complex, pixel-level thermal responses by leveraging data-driven approaches.
While data-driven approaches possess representational capabilities to perform this mapping, they are uninterpretable black boxes that may absorb experimental noise, prioritizing statistical correlation over thermodynamic plausibility.

Our case, however, is intrinsically more challenging due to the inherent measurement noise, the complexity of multimodal data, and the extreme spatial variations in wood materials.
Therefore, our predictive models need to balance model capacity, intra-species diversity, and physical regularization. 
The intrinsic intra-species diversity (characterized by spatially varying grain orientations, localized density, and unmodeled 3D subsurface structures) requires not only the high representational capacity of deep neural networks but also physical interpretability to predict wood thermal responses accurately. 
Yet, if we leave this representational capacity entirely unconstrained, that will cause the models to conflate genuine steady-state heat conduction with static measurement noise and unobservable 3D anomalies. 

\begin{figure}[t]
    \centering
    \includegraphics[width=\textwidth]{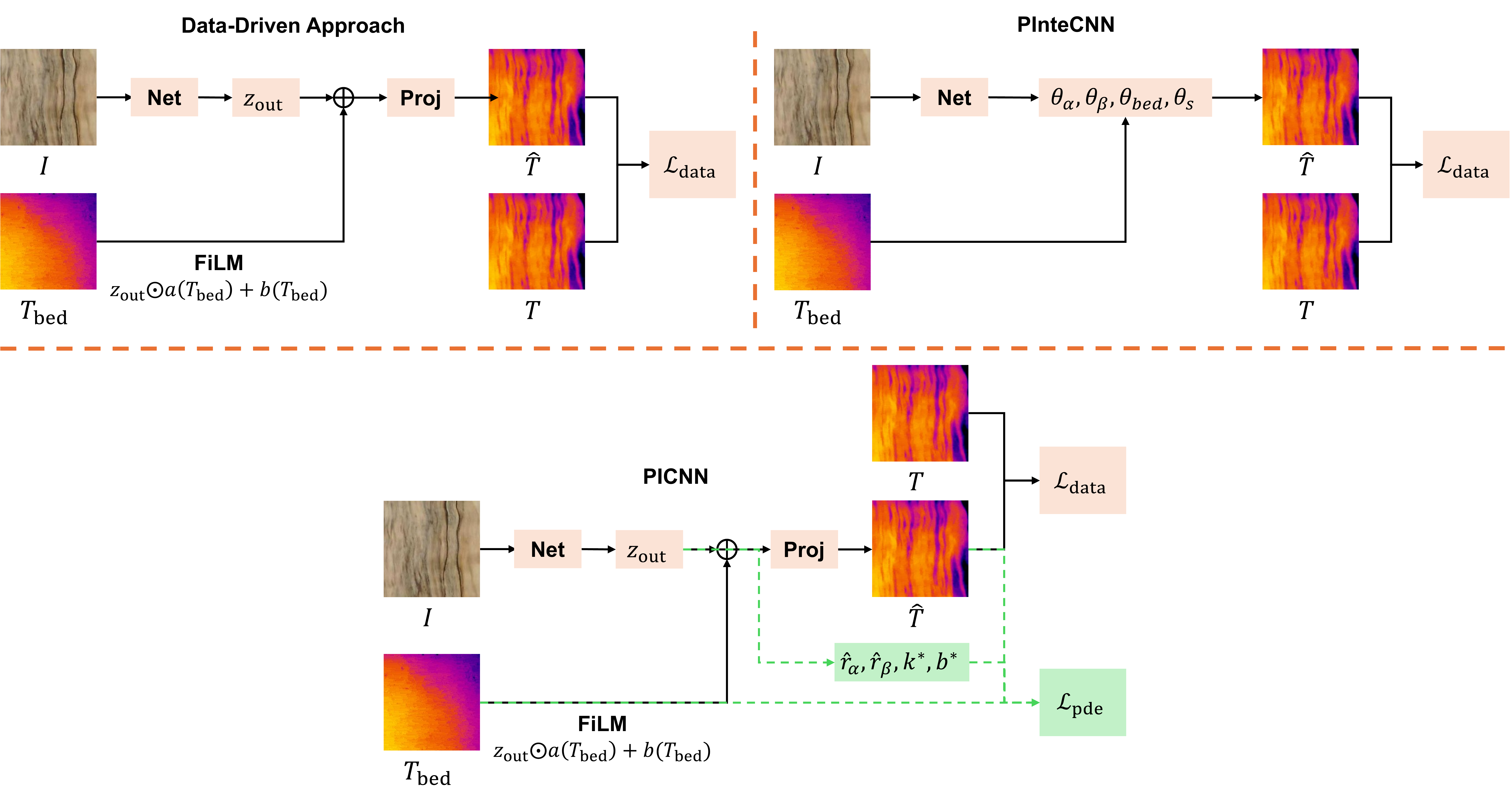} 
    \caption{
    Overview of three thermal prediction frameworks.
    Top Left: The data-driven approach processes the wood RGB image $I$, conditioned on the testbed temperature $T_{\text{bed}}$ via Feature-wise Linear Modulation (FiLM)~\cite{perez2018film}, to directly predict the thermal response $\hat{T}$ supervised solely by a data reconstruction loss function ($\mathcal{L}_{\text{data}}$). 
    Bottom: The PICNN builds upon the data-driven approach by incorporating soft physical constraints. 
    It supplements $\mathcal{L}_{\text{data}}$ with a physics-informed penalty loss function ($\mathcal{L}_{\text{pde}}$), governed by learned physical parameters ($\hat{r}_{\alpha}$, $\hat{r}_{\beta}$, $k^{*}$, $b^{*}$). 
    Top Right: The PInteCNN enforces hard physical constraints by explicitly embedding a numerical solver. 
    Instead of directly outputting $\hat{T}$, the network predicts spatially varying physical parameters ($\theta_{\alpha}$, $\theta_{\beta}$, $\theta_{\text{bed}}$, $\theta_{s}$), which the integrated solver subsequently processes to generate the final thermal response.
    }
    \label{fig:framework_overview}
\end{figure}

To address these limitations, this paper integrates deep learning and partial differential equations (PDEs) to predict wood thermal responses using wood RGB images and testbed temperature maps. 
Specifically, we present two distinct approaches to enforce a normalized 2D steady-state heat equation derived from the general heat transfer equation as a physical constraint in convolutional neural networks (Figure~\ref{fig:framework_overview}): 1) Physics-Informed Convolutional Neural Networks (PICNNs), which embed the thermodynamic law as a soft penalty term in the loss function, and 2) Physics-Integrated Convolutional Neural Networks(PInteCNNs), which hard-code the thermodynamic law directly into convolutional neural networks. 
Both approaches establish a thermodynamic bottleneck that acts as a robust regularizer. 
Enforcing these physical constraints by rejecting non-physical noise guarantees physically plausible generalization across intra-wood samples, which transforms convolutional neural networks into an interpretable framework capable of explicitly extracting underlying wood material properties, such as thermal anisotropy.
The main contributions of this paper are summarized as follows:
\begin{itemize}
    \item \textbf{Multimodal Data Acquisition and Cross-Modal Analysis.} 
We utilize an experimental setup to collect multiple wood datasets (Poplar, Grandis Cross-Cut (Grandis-CC), and Grandis Radial-Cut (Grandis-RC), as shown in Figure~\ref{fig:samples}). 
We identify an empirical inverse correlation indicating that visually darker regions of wood samples correspond to higher thermal conductance through cross-modal analysis. 
Additionally, deep feature space analysis using vision transformer-based models
demonstrates that intermediate network layers optimally capture the shared structural morphologies governing both wood visual appearances and thermal responses.

\item \textbf{Thermodynamic Model Formulation and Physics-Informed Architecture Design.} 
We derive a normalized 2D steady-state governing equation from fundamental 3D heat conduction principles, systematically incorporating specific boundary conditions (Dirichlet, Robin, and Neumann) and infrared camera radiance formulations to establish a rigorous mathematical mapping between testbed temperature maps and observed thermal responses.
We formulate two predictive models based on the normalized 2D steady-state governing equation. 
The PICNN incorporates physics through a soft PDE loss function, while the PInteCNN enforces physics via an explicitly embedded, differentiable numerical approximator-predictor-corrector solver.

\item \textbf{Physics-Stratified Data Partitioning and Quantitative Evaluation.} 
We introduce a rigorous bivariate sorting algorithm (Algorithm~\ref{alg:bivariate_split}) using Kolmogorov-Smirnov (KS) statistics to distribute structurally heterogeneous samples, thus mitigating domain shifts and ensuring statistical consistency across training, validation, and test datasets.
We demonstrate that PDE-based regularization improves predictive performance over data-driven approaches on our datasets.
Furthermore, the study characterizes the specific mathematical trade-offs among predictive accuracy, model interpretability, and intra-species diversity when applying hard or soft physical constraints.
\end{itemize}

\begin{figure}[t]
    \centering
    \includegraphics[width=1\textwidth]{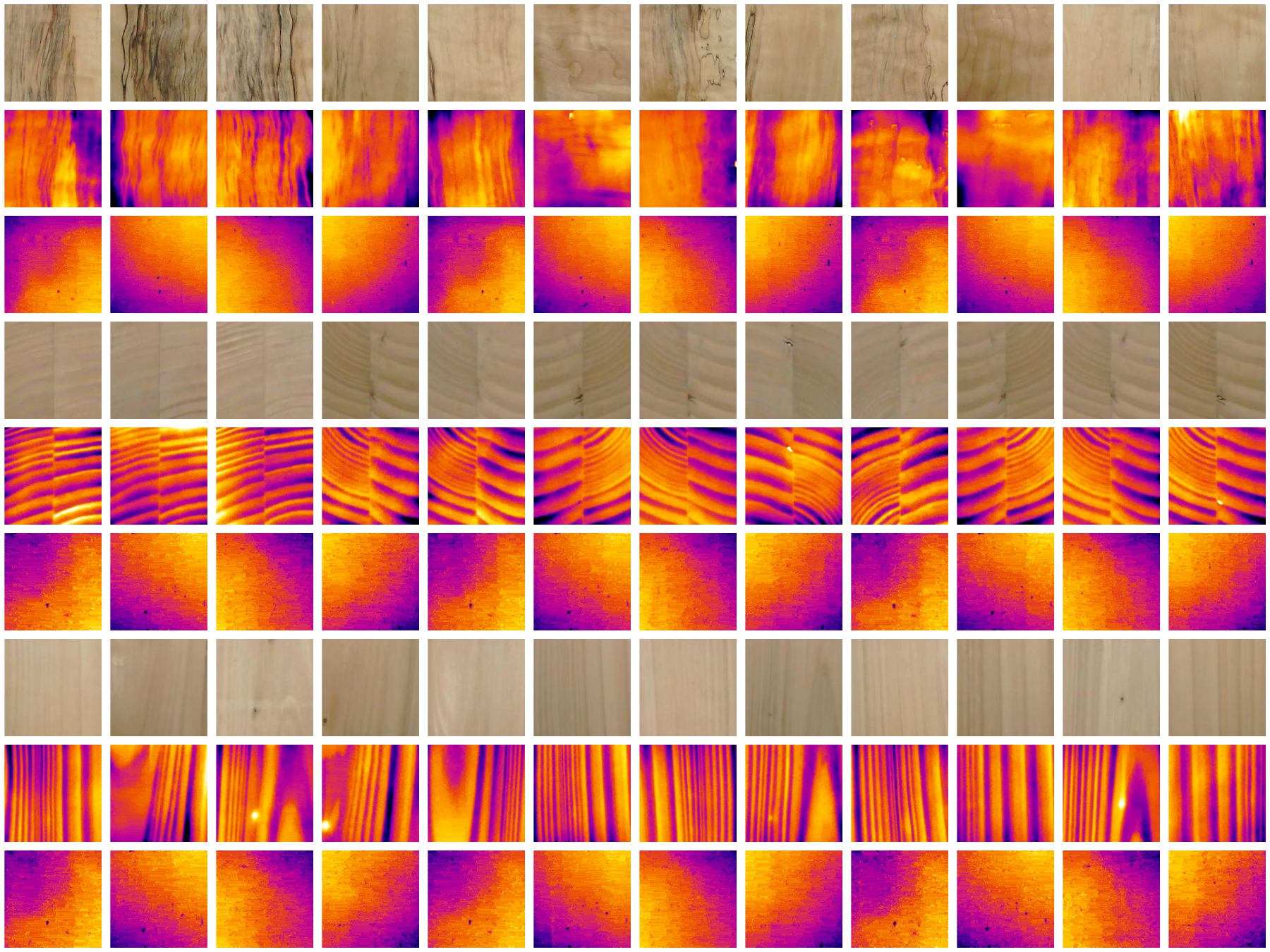}
    \caption{
    Selected data samples in the Poplar (top three rows), Grandis-CC (middle three rows), and Grandis-RC (bottom three rows) datasets. 
    For each category, the three rows display the wood RGB images, the corresponding thermal responses (normalized range $[-3, 3]$), and the testbed temperature maps (normalized range $[-3, 3]$).
    }
    \label{fig:samples}
\end{figure}

\section{Related Work}
\paragraph{Material Classification and Wood Identification from Thermal Imaging.}
Currently, the rapid development of deep learning has enabled image-level material classification from RGB or thermal images~\cite{johnson2011microgeometry, saponaro2015material, kampouris2016fine, cho2018deep, tanaka2018material, grossmann2022improving, dashpute2023thermal, herrera2024deep}.
For instance, thermal imagery has been effectively employed to classify materials at the image level~\cite{saponaro2015material}, with broader applications in wood identification extensively detailed in recent reviews~\cite{hwang2021computer, silva2022computer, wang2022applications}.
While these advanced deep learning models excel at material classification tasks, they are insufficient for analyzing highly anisotropic wood, which requires \textit{fine-grained pixel-to-pixel} analysis.
Some progress has been made in predicting wood properties.
For instance,~\cite{dashpute2023thermal} captured a stack of thermal images as the laser was switched on and off and then estimated the thermal parameters of a material, such as diffusivity and absorption parameters, using finite difference methods.
Despite these methodological advances, pixel-level analysis of thermal images of anisotropic wood remains largely unexplored.

\paragraph{Physics-Informed Convolutional Neural Networks.}
Incorporating physics into deep learning in scientific domains has led to the development of physics-informed or guided machine learning~\cite{raissi2019physics,karniadakis2021physics, yu2024learning, wang2025physics}, providing highly interpretable frameworks for complex physical systems.
One of the most elegant frameworks is Physics-Informed Neural Networks (PINNs)~\cite{raissi2017physics, raissi2019physics, jagtap2020conservative, cai2021physicsfluid, cai2021physics, chen2021physics}.
For instance, \cite{cai2021physics} introduced a PINN to solve heat transfer problems.
\cite{shaeri2025multimodal} integrated shortwave and longwave radiation modeling with deep learning to predict mean radiant temperature from multimodal environmental data.
Despite the potential of PINNs, incorporating spatial testbed temperature maps into standard PINN architectures remains difficult. 
Because PINNs typically depend on multi-layer perceptrons, they are inefficient at processing the spatially varying multimodal inputs required for our realistic cases.

Physics-Informed Convolutional Neural Networks (PICNNs) have emerged as a powerful extension of PINNs, achieving notable success in materials science~\cite{mianroodi2021teaching}, thermodynamics \cite{zhao2023physics, zhou2025automated}, fluid dynamics \cite{wandel2022spline, zhang2023physics, mohan2023embedding, liu2024multi, zhang2025mrf, zhou2025automated, pavlik2025fully}, and other fields \cite{zhu2019physics, gao2021phygeonet, rao2023encoding, fuhg2023deep}. 
For example, \cite{zhao2023physics} presented a PICNN for the steady-state temperature field prediction of heat source layouts, enabling physics-informed training without labeled data by incorporating heat conduction equations and boundary constraints directly into the loss function.
These works focus more on solving ideal PDEs without labeled data.
However, our goal is to predict wood thermal responses using noisy real-world data and PDEs.
To address this, we derive a 2D normalized governing equation and present two physics-informed frameworks for learning from noisy real-world observations while remaining strictly bounded by the 2D normalized governing equation.

\section{Thermal Data Collection}
\label{sec:thermal_data}

\noindent \textbf{Experimental Setup.}
Pixel-level thermal datasets for wood are scarce, thus presenting a significant bottleneck for analyzing, understanding, and predicting its thermal responses. 
To overcome this limitation, an experimental setup was designed to acquire paired wood RGB images and their corresponding thermal responses (Figure~\ref{fig:bed}).
The experimental setup consists of a temperature-controlled aluminum plate, an RGB camera, and an infrared camera (FLIR A700). 
The aluminum plate serves as a heat testbed, maintaining a boundary temperature map, $T_{\text{bed}}$. 
Both the RGB and infrared cameras are configured to capture corresponding wood RGB and thermal images at an identical spatial resolution of $480 \times 640$ pixels. 
Our primary objective is to measure the steady-state thermal responses of these wood samples.

\noindent \textbf{Measurement Protocol.}
Each wood sample is cut into a thin, square slab with dimensions $L_{x} \times L_{y} \times L_{z}$, where $L_{x} = L_{y} \gg L_{z}$ (e.g., $L_{x} = L_{y} = 100$ mm and $L_{z} = 3$ mm). 
The wood surface ($L_{x} \times L_{y}$) is captured by the RGB camera, while the corresponding thermal response on the surface is recorded by the infrared camera. 
The RGB images of the two opposite surfaces (Side A and Side B) of the wood sample are denoted as $I_{A}$ and $I_{B}$, respectively, and their corresponding surface thermal responses are denoted as $T_{A}$ and $T_{B}$.

During preliminary trials, we observed that thin wood samples are susceptible to slight natural deformation under thermal load, which can lead to inconsistent contact with the testbed. 
To ensure uniform thermal contact resistance across the sample-testbed interface, we apply a standardized metal weight to the top of the wood sample during a 5-minute heating phase, allowing the system to reach thermal equilibrium. 
The weight is then rapidly removed, and 10 consecutive frames of the thermal response are immediately recorded. 
While the removal of the weight introduces a brief transient phase, averaging these 10 frames effectively mitigates random sensor noise and minor transient artifacts, yielding a highly stable quasi-steady-state thermal response.

\begin{figure*}[t]
    \centering
    \includegraphics[width=1.0\textwidth]{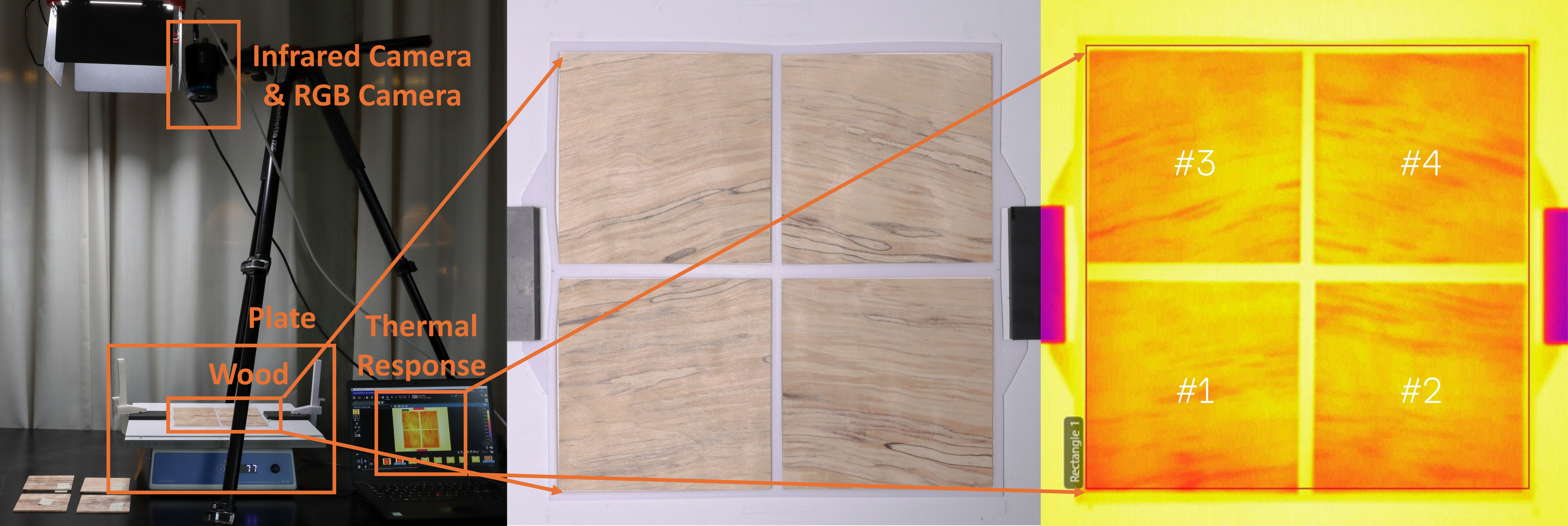}
    \caption{
    The experimental setup for thermal measurement. 
    The experimental setup consists of an aluminum testbed, an RGB camera, and an infrared camera. 
    Wood RGB images and their thermal responses of four wood samples are recorded simultaneously during each batch.
    }
    \label{fig:bed}
\end{figure*}

The measurement procedure follows a strictly standardized protocol: 1) record the testbed temperature, $T_{\text{bed}}$; 
2) apply the wood sample and weight until equilibrium is reached; 
3) measure the quasi-steady-state thermal response of the wood surface. 
Side A and Side B follow this identical protocol.
When measuring one side of a sample, the opposite side maintains direct thermal contact with the testbed, ensuring the bottom surface temperature strictly matches the testbed temperature, $T_{\text{bed}}$.

\noindent \textbf{Data Collection.} 
We collected data for three distinct wood species, including Poplar, Grandis-CC, and Grandis-RC (Figure~\ref{fig:samples}). 
To optimize data acquisition efficiency, we measured the samples in batches of four.
The Poplar dataset comprises 120 Poplar samples or 240 sides A/B.
The Grandis-CC dataset comprises 63 Grandis-CC samples or 126 sides A/B.
The Grandis-RC dataset comprises 53 Grandis-RC samples or 106 sides A/B. 
The paired wood RGB and thermal images were cropped into four individual ones for downstream analysis, understanding, and prediction. 
Specifically, the Poplar samples were cropped to a spatial resolution of $176 \times 176$ pixels, while the Grandis samples were cropped to $128 \times 128$ pixels.

\begin{figure}[t]
    \centering
    \includegraphics[width=1.0\textwidth]{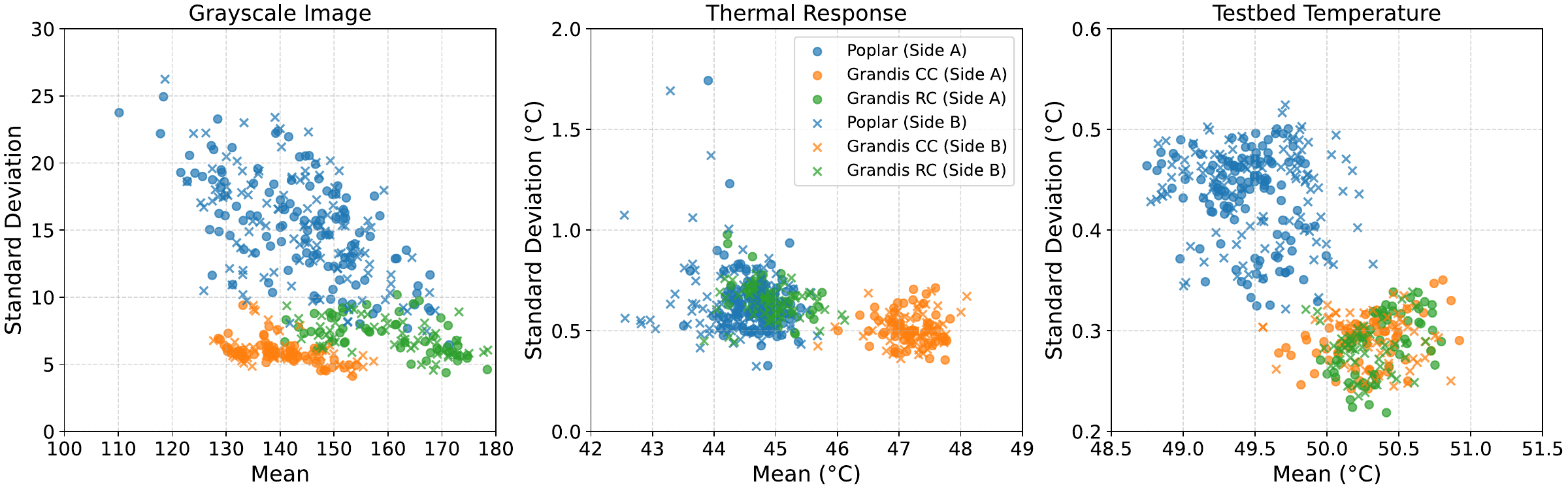}
    \caption{
    Sample-wise distributions of mean and standard deviation values for grayscale images (left), thermal responses (middle), and testbed temperatures (right) across the Poplar, Grandis-CC, and Grandis-RC datasets.
    }
    \label{fig:scatter_distributions}
\end{figure}

\noindent \textbf{Sample-wise Statistics.}
To investigate the inherent randomness of individual measurements within each wood species due to the anisotropy, we analyze sample-wise variability (Figure~\ref{fig:scatter_distributions}).
By plotting the mean against the standard deviation, these distributions highlight the strong statistical consistency between Side A and Side B while capturing the natural heterogeneity of the wood samples.
Specifically, the grayscale intensity reveals that Poplar exhibits high intra-sample visual variance (characterized by a broad spread of mean intensities and higher standard deviations), compared to the relatively uniform Grandis-CC and Grandis-RC.
Similarly, Grandis-CC forms an isolated cluster with the highest mean thermal response and lowest variability.
In contrast, Poplar and Grandis-RC show lower mean thermal response, with Poplar notably containing several high-variance outliers.
Finally, testbed temperature analysis indicates that the heat source beneath Poplar is generally cooler and more variable, while it remains higher and more stable beneath both Grandis-CC and Grandis-RC.
Additional detailed analyses are provided in Appendix~\ref{sec:wood_analysis_app}.

\begin{figure}[t]
    \centering
    \includegraphics[width=1\textwidth]{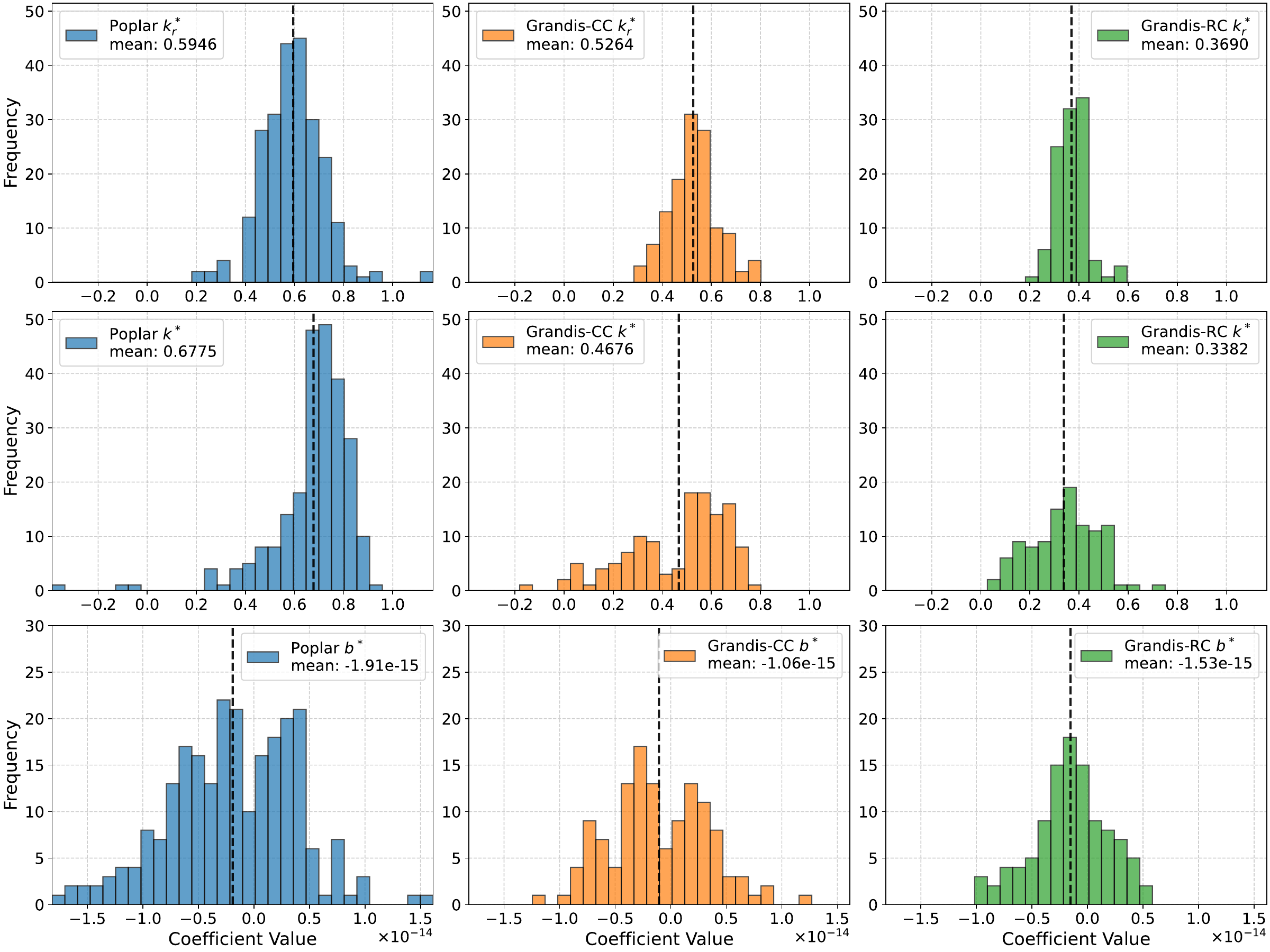}
    \caption{
    Sample-wise distributions of the thermal-to-testbed mapping coefficients $k^*_r$, $k^*$, and $b^*$  across the (Poplar (left), Grandis-CC (middle), and Grandis-RC (right)) datasets.
    }
    \label{fig:coefficient}
\end{figure}
\section{Method} 
\label{sec:physics_informed_learning}
Physics-informed learning aims to embed physical laws directly into a neural network or its loss function. 
To achieve this, \textit{the key idea is to establish a mathematical connection between the physical parameters in the governing equations and the learned feature representations from wood RGB images.}

\noindent \textbf{Wood Thermal Modeling.}
We provide a simplified model in the following and elaborate on the wood thermal modeling in Appendices~\ref{sec:wood_thermal_modeling} and~\ref{sec:re_pde_app}.
To account for individual wood sample variability (Figure~\ref{fig:scatter_distributions}), each thermal response or testbed temperature map is standardized by its own spatial mean and standard deviation, defined as $\tilde{T}^* = (\tilde{T} - \mu) / \sigma$
\begin{align}
    \label{eq:norm_pde_2d_discrete}
    \tilde{T}^*_{\text{mea}} = \underbrace{\left( \frac{ r_{\gamma} \sigma_{\text{bed}}}{\sigma_{\text{mea}}} \right)}_{k^*} \tilde{T}^*_{\text{bed}} + \underbrace{\left( \frac{r_{\gamma} \mu_{\text{bed}} - \mu_{\text{mea}}}{\sigma_{\text{mea}}} \right)}_{b^*} + \underbrace{r_{\alpha} L_{z}^{2} \frac{\partial^2 \tilde{T}^*_{\text{mea}}}{\partial x^2} + r_{\beta} L_{z}^{2} \frac{\partial^2 \tilde{T}^*_{\text{mea}}}{\partial y^2} }_{\text{Spatial Diffusion}},
\end{align}
where $\alpha = k_{x} / k_{z}$, $\beta = k_{y}/k_{z}$, and $\text{Bi} = h L_{z} / k_{z}$ is a Biot number.
We can learn the ratios $r_{\alpha}$, $r_{\beta}$, and $r_{\gamma}$ with physical bounds $r_{\alpha} = \frac{\alpha}{1 + \text{Bi}} > 0, r_{\beta} = \frac{\beta }{1 + \text{Bi}} > 0, r_{\gamma} = \frac{\gamma}{1 + \text{Bi}} > 0$.

The spatial diffusion term represents lateral heat transfer governed by the anisotropic diffusivities ($\alpha$, $ \beta$) inherent to wood anisotropy.
When it reaches the steady state ($b^*=0$, $k^* = \mu_{\text{mea}} \sigma_{\text{bed}} / \mu_{\text{bed}} \sigma_{\text{mea}}$), we have a very special governing equation
\begin{align} 
    \tilde{T}^*_{\text{mea}} = \underbrace{\left( \frac{\mu_{\text{mea}} \sigma_{\text{bed}}}{\sigma_{\text{mea}} \mu_{\text{bed}} } \right)}_{k^*_r} \tilde{T}^*_{\text{bed}} + \underbrace{ r_{\alpha} L_{z}^{2} \frac{\partial^2 \tilde{T}^*_{\text{mea}}}{\partial x^2} + r_{\beta} L_{z}^{2}\frac{\partial^2 \tilde{T}^*_{\text{mea}}}{\partial y^2} }_{\text{Spatial Diffusion}}.
\end{align}

The normalized governing equation~\ref{eq:norm_pde_2d_discrete} is the key equation for analysis and prediction.
In this framework, the slope $k^*$ serves as a dimensionless spatial transfer factor that quantifies the efficiency of heat conduction through the wood thickness. 
The intercept $b^*$ represents the global energy balance residual.
Our empirical analysis (Figure~\ref{fig:coefficient}) confirms $ b^* \approx 0$, indicating that the wood samples consistently reach a global mean equilibrium, where $\mu_{\text{mea}} \approx r_{\gamma} \mu_{\text{bed}}$. 
The coefficients $k^*$ and $k^*_r$ characterize the linear mapping between the testbed temperature $T^*_{\text{bed}}$ and the thermal measurement $T^*_{\text{mea}}$ but represent distinct physical perspectives.
The $k^*$ coefficient quantifies spatial signal preservation, specifically measuring how much of the testbed's intrinsic spatial temperature distribution remains identifiable after undergoing internal diffusion within the wood's structure.
In contrast, $k^*_r$ is calculated from raw magnitudes as the ratio of Coefficients of Variation ($\text{CV}_{\text{bed}}/\text{CV}_{\text{mea}}$).
However, the observed variance (spread) within these $k^*$ distributions (Figure~\ref{fig:coefficient}) poses a significant challenge for model generalization. 
While the tight clustering confirms that the wood samples follow a shared physical law, the internal spread reflects the anisotropic nature of the wood samples, such as localized density and internal structure.

\noindent \textbf{Physics-Informed Convolutional Neural Networks.}
To make Equation~\ref{eq:norm_pde_2d_discrete} suitable for optimizing models via soft regularization in convolutional neural networks, namely Physics-Informed Convolutional Neural Networks (PICNNs), it is discretized using a finite central difference method
\begin{align}
    \label{eq:2d_anisotropic_normalized_discrete}
    R_{i,j} = &\hat{r}_{\alpha, i,j} \left( \tilde{T}^{*}_{\text{mea},i-1,j} - 2\tilde{T}^{*}_{\text{mea},i,j} + \tilde{T}^{*}_{\text{mea},i+1,j} \right) + \hat{r}_{\beta, i,j} \left( \tilde{T}^{*}_{\text{mea},i,j-1} - 2\tilde{T}^{*}_{\text{mea},i,j} + \tilde{T}^{*}_{\text{mea},i,j+1} \right) \notag \\
    &- \tilde{T}^{*}_{\text{mea},i,j} + k^{*}_{i,j} \tilde{T}^{*}_{\text{bed},i,j} + b^{*}_{i,j},
\end{align}
where $\hat{r}_{\alpha,i,j} = r_{\alpha,i,j} L_{z}^2 / (\Delta x)^2$ and $\hat{r}_{\beta,i,j} = r_{\beta,i,j} L_{z}^2 / (\Delta y)^2$.
We learn $\hat{r}_{\alpha}$, $\hat{r}_{\beta}$, and $k^{*}$ using softplus activation, $f(x) = \ln (1 + \exp(x))$, and $b^{*}$ without any activation.

Following physics-informed neural networks~\cite{raissi2017physics, raissi2019physics}, we incorporate the PDE residual as a soft regularization term, i.e., $\mathcal{L} = \mathcal{L}_{\text{data}} + \lambda_{\text{pde}} \mathcal{L}_{\text{pde}}$, where $\lambda_{\text{pde}}$ is a weighting coefficient balancing the data fidelity and physical consistency. 
To address measurement magnitude variability across individual wood samples (Figure~\ref{fig:scatter_distributions}) and mitigate domain shift, we employ a scale- and shift-invariant loss function~\cite{ranftl2020towards, yang2024depth, yang2024depthv2}. 
The predicted thermal response, $\hat{T}_{\text{mea}}$, is aligned with the ground truth, $T_{\text{mea}}$, using scale $s(T_{\text{mea}})$ and shift $t(T_{\text{mea}})$ factors, i.e., $\mathcal{L}_{\text{data}}(T_{\text{mea}}, \hat{T}_{\text{mea}}) = \dfrac{1}{HW} \sum_{i=1}^{HW} \left| s \hat{T}_{\text{mea}, i}^{*} + t - T_{\text{mea},i} \right|$, where $t = \text{mean}(T_{\text{mea}}), \quad s = \sqrt{\dfrac{1}{HW} \sum_{i=1}^{HW} \left(T_{\text{mea},i} - t \right)^2}$.
The PDE residual loss is defined as the mean squared residual over the spatial domain $\Omega$, i.e., $\mathcal{L}_{\text{pde}} = \frac{1}{|\Omega|} \sum_{(i,j) \in \Omega} \left| R_{i,j} \right|^2$.
Because operations in the residual Equation~\ref{eq:2d_anisotropic_normalized_discrete} are fully differentiable, they can be seamlessly embedded into frameworks such as PyTorch~\cite{paszke2019pytorch},  optimizing the parameters via automatic differentiation~\cite{baydin2018automatic}.

\noindent \textbf{Physics-Integrated Convolutional Neural Networks.}
The governing physical laws can be directly embedded into convolutional neural networks as hard constraints, namely Physics-Integrated Convolutional Neural Networks (PInteCNNs). 
We can formulate this general learned mapping as
\begin{equation}
\label{eq:pde_2d_general}
\tilde{T}_{\text{mea}}^{*} = \theta_{\text{bed}} \tilde{T}_{\text{bed}}^{*} + \theta_{s} + \theta_{\alpha} \frac{\partial^2 \tilde{T}_{\text{mea}}^{*} }{\partial x^2} + \theta_{\beta} \frac{\partial^2 \tilde{T}_{\text{mea}}^{*}}{\partial y^2},
\end{equation}
where the PInteCNN processes the wood RGB image to dynamically predict the spatially varying physical parameters $(\theta_{\alpha}, \theta_{\beta}, \theta_{\text{bed}}, \theta_{s})$.
Because the target thermal response, $\tilde{T}_{\text{mea}}^{*}$, appears on both sides of Equation~\ref{eq:pde_2d_general}, we solve it using an explicit approximator-predictor-corrector solver executed directly within the forward pass
\begin{align}
    \text{Approximator ($k=0$):} &\quad \tilde{T}_{0} = \theta_{\text{bed}} \tilde{T}_{\text{bed}}^{*} + \theta_{s}, \\
    \text{Predictor ($k=1$):} &\quad \tilde{T}_{1} = \tilde{T}_{0} + \theta_{\alpha} \frac{\partial^2 \tilde{T}_{0} }{\partial x^2} + \theta_{\beta} \frac{\partial^2 \tilde{T}_{0} }{\partial y^2}, \label{eq:predictor}\\
    \text{Corrector ($k \in [2, K]$)}: &\quad \tilde{T}_{k} = \tilde{T}_{0} + \theta_{\alpha} \frac{\partial^2 \tilde{T}_{k-1} }{\partial x^2} + \theta_{\beta} \frac{\partial^2 \tilde{T}_{k-1} }{\partial y^2}, \label{eq:corrector_k}\\
    \text{Final Output}: &\quad \tilde{T}_{\text{mea}}^{*} = \tilde{T}_{K}. \label{eq:final_output}
\end{align}

By strictly enforcing the governing physics in the PInteCNN's internal operations and unrolling the numerical solver for $K$ steps, this architecture can be effectively trained by optimizing only the data-fidelity loss, $\mathcal{L}_{\text{data}}$.
More information is detailed in Appendix~\ref{sec:architectures}.

\noindent \textbf{Data Partition Strategy.}
Partitioning small, highly heterogeneous physical datasets poses significant risks of domain shift. 
To ensure a rigorously unbiased evaluation, we design a physics-stratified partitioning strategy (Algorithm~\ref{alg:bivariate_split} in Appendix~\ref{sec:proof_bivariate}). 
First, to prevent data leakage, Side A and Side B of the same wood sample are strictly grouped. 
Second, to mitigate domain shift, we evaluate $k^*$ for each side via linear regression. 
We enforce a strict threshold ($\tau = 0.05$). 
Any sample exhibiting $k^* < \tau$ on either side is discarded as an experimental anomaly, ensuring the neural network only learns from valid steady-state physical mappings from thermal responses.

Because valid wood samples exhibit significant structural heterogeneity, heat transfer can still differ substantially between opposite sides. 
Therefore, we represent each retained wood sample using two physical metrics, including mean transfer coefficient $\overline{k}^{*} = (k^*_A + k^*_B) / 2$ and structural asymmetry $\Delta k^* = |k^*_A - k^*_B|$. 
The valid wood samples are sorted by these two physical metrics and distributed sequentially into $M=10$ balanced subsets using a systematic, shifted round-robin allocation. 

While this systematic allocation guarantees that subset statistics converge to the global distribution at a rate of $\mathcal{O}(1/P)$ (where $P$ is the number of complete batches), the limited size of physical datasets often results in a small $P$, which can still yield a minor empirical distribution shift. 
To strictly minimize this residual domain shift, we treat the final split construction as a combinatorial optimization problem. 
Rather than employing a static assignment, we evaluate all possible combinations for assigning the $M=10$ subsets into a 6:2:2 training, validation, and test split. 
For each candidate assignment, we compute the two-sample Kolmogorov-Smirnov (KS) statistics~\cite{massey1951kolmogorov, hodges1958significance, smirnov1939estimate, rabanser2019failing} to quantify the empirical distribution distances of $\overline{k}^{*}$ and $\Delta k^*$ between the generated splits. 
The combination that yields the minimum aggregate KS distance is selected, ensuring the final training, validation, and test sets share virtually identical physical distributions (Appendix~\ref{sec:proof_bivariate}).

\section{Experiments and Results}
\label{sec:exp}
This section compares the predictive performance of the three thermal prediction approaches. 
The data-driven approach serves as a baseline for benchmarking the improvements achieved by incorporating soft physical constraints (PICNN) and hard physical solvers (PInteCNN).
The comprehensive quantitative results of the three approaches are summarized in Table~\ref{tab:whole_model_evaluation}, and qualitative results are presented in Figures~\ref{fig:best_qualitative_results} and~\ref{fig:worst_qualitative_results}.
More results and discussions are provided in Appendices~\ref{sec:ablation} and \ref{sec:data_integration}.

\subsection{Experimental Setup}
\noindent \textbf{Training Data.}
We evaluate our physics-informed deep learning approaches on three real-world wood thermal datasets using the data partition strategy in Algorithm~\ref{alg:bivariate_split}.
We partition the datasets after filtering out potential anomalies and strictly group Side A and Side B from the same wood sample to prevent data leakage.
Each dataset is partitioned into training, validation, and test sub-datasets in a 6:2:2 ratio.

\noindent \textbf{Training Details.}
Given the unique characteristics of the three datasets (Figures~\ref{fig:scatter_distributions} and~\ref{fig:coefficient}), we train three distinct models.
All these models are trained on H100 using the Adam optimizer~\cite{kingma2014adam} with a batch size of $32$.
Instead of using standard 32-bit floating-point (FP32) arithmetic for training models, we use 64-bit floating-point (FP64) arithmetic~\cite{xu2025fp64}.
The learning rate is initially set to $0.01$ using an early-stop mechanism.
Several data augmentations, including horizontal flipping and vertical flipping, are applied to reduce overfitting during model training.

\noindent \textbf{Evaluation Metrics.}
Following~\cite{ranftl2020towards}, we first scale and shift the predicted thermal response to align with the true thermal response.
We then employ two widely used metrics to evaluate the quality of thermal response prediction.
One is the Mean Absolute Error (MAE) metric, i.e., $\text{MAE} = 1/HW \sum_{i=1}^{HW} |T_{i} - \hat{T}_{i}|$.
The other is the Root Mean Squared Error (RMSE) metric, i.e., $\text{RMSE} = \sqrt{1/HW \sum_{i=1}^{HW} \left(T_{i} - \hat{T}_{i} \right)^2}$.
To further evaluate the prediction accuracy, we introduce the percentage of pixels, i.e., $\delta_{01} = \text{max}(T_{i}/\hat{T}_{i}, \hat{T}_{i}/T_{i}) < 1.01$.
We adopt much stricter thresholds than~\cite{ranftl2020towards, yang2024depth, yang2024depthv2} because our thermal prediction achieves higher prediction accuracy.

\subsection{Results on PICNN and {PInteNN}}

\noindent \textbf{The Necessity of Domain-Specific Modeling.}
To comprehensively evaluate the data-driven approach, we measure its predictive performance—quantified by MAE, RMSE, and $\delta_{01}$—across two distinct modeling paradigms: a globally optimized unified model versus three domain-specific models. As illustrated in Table~\ref{tab:whole_model_evaluation}, the domain-specific models consistently yield superior predictive accuracy. Across the Poplar, Grandis-CC, and Grandis-RC datasets, these tailored models achieve strictly lower MAE and RMSE, alongside higher $\delta_{01}$ scores, than the unified baseline. While the unified model successfully captures broad, macro-level thermal trends, its generalized capacity is insufficient to fully adapt to the unique distributional characteristics and structural nuances of individual wood species. Consequently, training domain-specific models remains an essential requirement for minimizing localized predictive errors and achieving optimal, high-fidelity performance across diverse datasets.

\begin{table}[t]
  \caption{
  Quantitative results of the three thermal prediction approaches. 
  The data-driven approach, PICNN, and PInteCNN are evaluated across the Poplar, Grandis-CC, and Grandis-RC datasets using MAE, RMSE, and $\delta_{01}$. 
  Bold values denote the best performance.
  }
  \label{tab:whole_model_evaluation}
  \centering
  \resizebox{\textwidth}{!}{%
  \begin{tabular}{l ccc ccc ccc}
    \toprule
    \multirow{2}{*}{\textbf{Method}} & \multicolumn{3}{c}{\textbf{Poplar}} & \multicolumn{3}{c}{\textbf{Grandis-CC}} & \multicolumn{3}{c}{\textbf{Grandis-RC}} \\
    \cmidrule(lr){2-4} \cmidrule(lr){5-7} \cmidrule(lr){8-10}
    & MAE $\downarrow$ & RMSE $\downarrow$ & $\delta_{01}$ (\%) $\uparrow$ & MAE $\downarrow$ & RMSE $\downarrow$ & $\delta_{01}$ (\%) $\uparrow$ & MAE $\downarrow$ & RMSE $\downarrow$ & $\delta_{01}$ (\%) $\uparrow$ \\
    \midrule
    Data-Driven (Unified)  & 0.3330 & 0.4301 & 72.70 & 0.2512 & 0.3142 & 86.80 & 0.3817 & 0.4721 & 65.27 \\
    Data-Driven (Specific) & 0.3116 & 0.4023 & 75.64 & 0.2051 & 0.2631 & 92.05 & 0.3578 & 0.4526 & 69.00 \\
    PICNN       & \textbf{0.3088} & \textbf{0.3974} & \textbf{76.05} & \textbf{0.2043} & \textbf{0.2619} & \textbf{92.13} & \textbf{0.3456} & \textbf{0.4386} & \textbf{71.57} \\
    PInteCNN    & 0.3096 & 0.3992 & 75.82 & 0.2101 & 0.2694 & 91.35 & 0.3497 & 0.4387 & 70.17 \\
    \bottomrule
  \end{tabular}%
  }
\end{table}

\noindent \textbf{Data-Driven Approach's Performance.}
The data-driven approach performs optimally on the Grandis-CC dataset, achieving the lowest error (MAE of $0.2051$, RMSE of $0.2631$) and the highest prediction accuracy ($\delta_{01}$ of $92.05\%$).
Conversely, it exhibits significant performance degradation on the Grandis-RC dataset (MAE of $0.3578$, RMSE of $0.4526$, and $\delta_{01}$ of $69.00\%$).
The Poplar dataset yields intermediate performance (MAE of $0.3116$, RMSE of $0.4023$, and $\delta_{01}$ of $75.64\%$).
This discrepancy indicates that the data-driven approach is highly sensitive to the structural heterogeneity and complex thermal diffusion patterns inherent to spatially varying wood species.

\noindent \textbf{Impact of Soft Physical Constraints (PICNN).}
According to Table~\ref{tab:model_evaluation}, for the Grandis-CC dataset, the PICNN achieves peak performance at $\lambda_{\text{pde}} = 10^{-1}$ (MAE of $0.2043$), outperforming the baseline. 
Similar improvements occur in the Poplar and Grandis-RC datasets, where the optimal configurations ($\lambda_{\text{pde}} = 10^{-5}$ and $\lambda_{\text{pde}} = 10^{-6}$) reduce the MAEs to $0.3088$ and $0.3456$, respectively.
These gains demonstrate that soft PDE constraints help the convolutional neural network capture underlying thermodynamics.

\noindent \textbf{Impact of Hard Physical Constraints (PInteCNN).}
Based on Table~\ref{tab:model_evaluation}, the PInteCNN demonstrates highly competitive results, particularly at lower corrector step iterations.
Applying one corrector step on the Poplar, Grandis-CC, and Grandis-RC datasets achieves optimal MAEs of $0.3096$, $0.2101$, and $0.3497$, respectively.
However, increasing the corrector steps to $2$ consistently degrades predictive performance across all datasets.
For example, the MAE increases substantially to $0.2824$ on the Grandis-CC dataset. 
This trend reveals a fundamental trade-off inherent to integrating explicit numerical solvers.
While embedding a hard physics solver is advantageous for structural regularization, the numerical evaluation of the Laplacian operator inherently amplifies high-frequency noise present in these textures, causing errors to compound with each successive solver iteration.

\subsection{Interpretation of Physics-Informed Network Parameters}
\noindent \textbf{Parameter Convergence and Stability in PICNN.}
Part I of Table~\ref{tab:merged_picnn_pintecnn} details the convergence of the governing physical parameters ($\hat{r}_{\alpha}$, $\hat{r}_{\beta}$, $k^{*}$, $b^{*}$) for the PICNN. 
A primary observation is the distinct consistency of these parameters across the training, validation, and test splits for any given $\lambda_{\text{pde}}$. 
This stability demonstrates that the PICNN successfully learns generalized physical representations rather than overfitting to the training distribution.
It also demonstrates our data partition strategy (Algorithm~\ref{alg:bivariate_split}).
Furthermore, the magnitude of the penalty weight $\lambda_{\text{pde}}$ directly dictates the physical parameters to which the PICNN converges.  
This highlights how the magnitude of the penalty weight $\lambda_{\text{pde}}$ dictates the fundamental trade-off between data-driven feature extraction and strict adherence to physical conservation laws.

\begin{figure}[t]
  \centering
  \includegraphics[width=\textwidth]{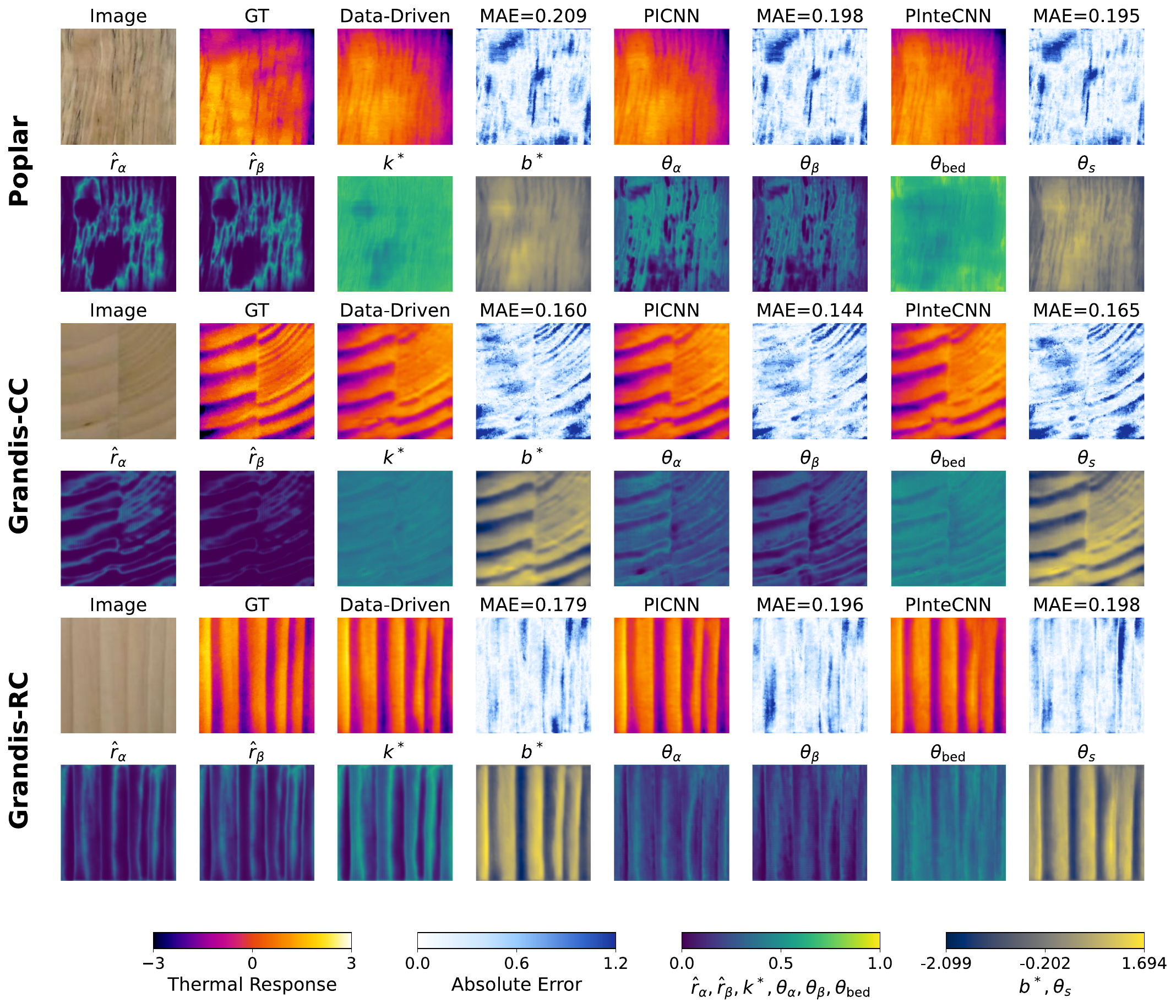} 
  \caption{
  Qualitative results of the best-performing test samples (minimum MAE) across the Poplar, Grandis-CC, and Grandis-RC datasets.
  }
  \label{fig:best_qualitative_results}
\end{figure}

\begin{figure}[t]
  \centering
  \includegraphics[width=\textwidth]{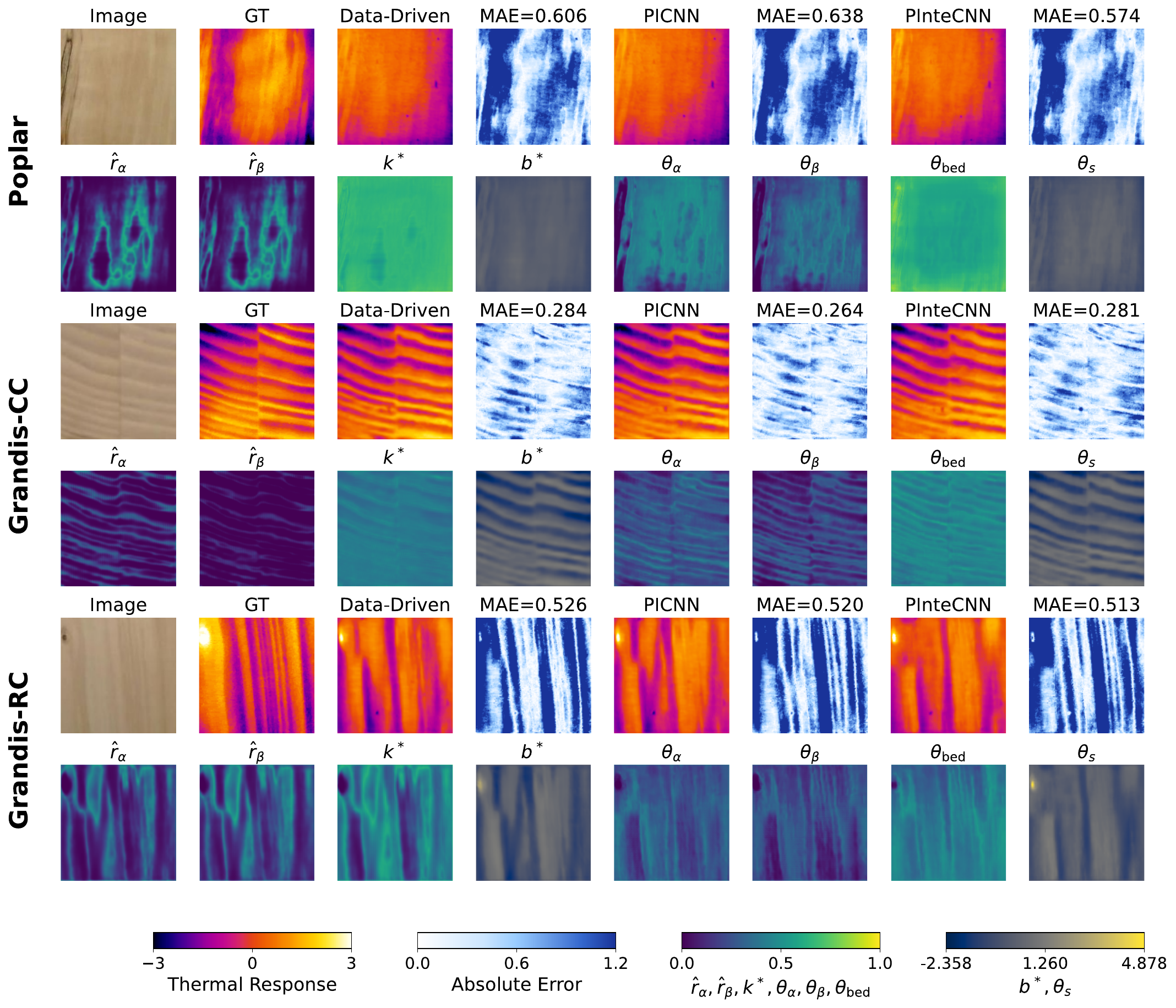}
  \caption{
  Qualitative results of the worst-performing test samples (maximum MAE) across the Poplar, Grandis-CC, and Grandis-RC datasets.}
  \label{fig:worst_qualitative_results}
\end{figure}

\noindent \textbf{Parameter Dynamics and Solver Integration in PInteCNN.}
Part II extends this physical interpretation to the PInteCNN, which features a hard-coded physical solver. 
Consistent with the PICNN, the spatially averaged parameters ($\theta_{\alpha}$, $\theta_{\beta}$, $\theta_{\text{bed}}$, $\theta_{\text{s}}$) exhibit robust generalization, remaining stable across data splits.  
While the one-step configuration maintains minimum error, increasing the computational burden to $2$ steps consistently degrades predictive accuracy. 
Forcing the PInteCNN to execute $2$ steps causes the parameters to fluctuate unpredictably (e.g., in Grandis-CC, $\theta_{\alpha}$ rises back to $0.29$ and $\theta_{\text{bed}}$ spikes to $0.56$). 
This parameter instability directly leads to error accumulation and numerical divergence when mapping high-contrast wood textures across multiple hard-coded solver iterations.

\subsection{Analysis of Topological Variance and Model Stability}
To comprehensively evaluate model robustness, we conduct a sequential data integration ablation study (Appendix~\ref{sec:data_integration}) that analyzes the expected performance (Main Path) across statistics-matched (Best-First), diversity-driven (Worst-First), and stochastic (Random) training trajectories.
This analysis reveals a consistent scaling paradigm across all three datasets. 
Under extreme data scarcity, the data-driven approach demonstrates a notable capacity to rapidly adapt to localized distributions, achieving competitive initial error rates compared to physics-informed baselines. 
As the integration sequence progresses, however, this data-driven approach frequently experiences diminishing returns and forms distinct performance plateaus (as observed in Poplar and Grandis-RC). 
In contrast, the physics-informed models—particularly the PICNN—exhibit more robust asymptotic scaling across the evaluated dataset topologies.
While embedding physical constraints can initially exacerbate sequence-dependent variance under extreme sparsity (as observed in the early stages of Grandis-RC), it ultimately provides the structural scaffolding required to bypass premature statistical plateaus (Poplar), sustain steeper learning trajectories against diminishing statistical returns (Grandis-CC), and map complex structural boundaries as data volume scales. 
Ultimately, these expected performance trajectories empirically demonstrate that while statistical architectures excel at early sparse learning, physical regularization provides a critical mechanism to bypass plateaus and sustain continuous improvement.

\noindent \textbf{Balancing Generalizability, Interpretability, and Diversity.}
The integration of physics into deep learning should balance among model generalizability, physical interpretability, and the capability of handling diverse wood samples within a given wood type.
Individual wood samples within a given wood type exhibit unique thermal variability due to their natural anisotropy (Figures~\ref{fig:scatter_distributions} and~\ref{fig:coefficient}).
The data-driven approach can provide empirical solutions on specific datasets by learning an implicit mapping directly from wood RGB images and testbed temperature maps.
However, their black-box nature obscures the underlying thermodynamic mechanisms. 
The PICNN and PInteCNN introduce domain-specific inductive biases, guiding predictions toward established heat transfer principles. 
While the PICNN implements this through soft regularization ($\lambda_{\text{pde}}$), the PInteCNN establishes a more interpretable structure by explicitly predicting spatially varying physical parameters ($\theta_{\alpha}, \theta_{\beta}, \theta_{\text{bed}}, \theta_{s}$) that govern an embedded numerical solver.
We adjust the PDE penalty in the PICNN or constrain the solver iterations in the PInteCNN. 
This provides a practical means to balance strict physical interpretability with the flexibility needed to approximate the distinct thermal behaviors of individual wood samples under limited data constraints.

\noindent \textbf{Limitations of 2D Spatial Simplification.}
A primary limitation of the current physical formulation is its dependence on a two-dimensional spatial simplification. 
We approximate the wood sample as a 2D plane and model depth-wise heat transfer as a linear transformation of the testbed temperature, ignoring the three-dimensional internal structure of the wood sample.
Subsurface volumetric features, such as internal knots and localized density, can alter surface thermal diffusion. 
Because the 2D governing equation inherently lacks the capacity to represent these depth-dependent thermodynamic interactions, the resulting thermal predictions may fail to capture highly localized surface anomalies. 
Future research could explore the integration of 3D transient heat conduction solvers during neural network training.

\section{Conclusion}
\label{sec:conclusion}
This study leverages physics-informed deep learning frameworks to analyze, understand, and predict wood thermal responses.
Through evaluations across three real-world wood datasets, we demonstrate that the incorporation of physical constraints alters model behavior, presenting distinct trade-offs among predictive accuracy, computational stability, and physical interpretability.
Our findings indicate that while data-driven approaches can yield reasonable metrics, they exhibit performance degradation when confronted with structural heterogeneity in the limited-data setting. 
To address this limitation, the Physics-Informed Convolutional Neural Network (PICNN) leverages a soft PDE penalty. 
This constraint acts as a powerful regularizer, guiding the network to a comparatively optimal predictive accuracy. 
The Physics-Integrated Convolutional Neural Network (PInteCNN) explicitly embeds an iterative numerical solver within its forward pass. 
This architecture provides a highly interpretable mechanism by extracting explicit spatial physical parameter maps and delegating the diffusion steps to a mathematical solver.
Embedding physical inductive biases into convolutional neural networks is a highly effective strategy for overcoming data limitations. 
Ultimately, both the PICNN and PInteCNN successfully enhance empirical performance and scientific interpretability.

\section*{Acknowledgment}
This study is financially supported by the Villum Synergy Grant (Grant No. 57401).

\bibliography{egbib}

\appendix
\section{ Wood Thermal Analysis}
\label{sec:wood_analysis_app}

\paragraph{Overall Sample Statistics.}
Table \ref{tab:compact_variation_stats} summarizes the overall statistics for grayscale image intensities, testbed temperatures, and thermal responses for the three wood species across both surfaces (Side A and Side B). 
We characterize their overall distributions using mean $\mu$ and standard deviation $\sigma$ values.
The testbed temperatures remain highly consistent, averaging near $50^\circ \mathrm{C}$ across all wood species.
Thermally, the average thermal responses distinctly separate Grandis-CC ($47.2 ^\circ \mathrm{C}$) from Grandis-RC ($44.9^\circ \mathrm{C}$) and Poplar ($44.5^\circ \mathrm{C}$), which share more heavily overlapping thermal profiles (Figure~\ref{fig:overall_histograms}). 
Visually, the grayscale image intensities further differentiate the wood species.
Grandis-RC displays the highest mean intensity ($\mu \approx 160$), while Poplar demonstrates the broadest distribution ($\sigma \approx 19.5$), capturing its wider natural variance. 
Across all modalities, the data highlights a strong statistical consistency between Side A and Side B.

\paragraph{Sample-wise Statistics.}
To investigate the inherent randomness of individual measurements within each wood species due to the anisotropy, we analyze sample-wise variability (Figure~\ref{fig:scatter_distributions}).
By plotting the mean against the standard deviation, these distributions highlight the strong statistical consistency between Side A and Side B while capturing the natural heterogeneity of the wood samples.
Specifically, the grayscale intensity reveals that Poplar exhibits high intra-sample visual variance (characterized by a broad spread of mean intensities and higher standard deviations), compared to the relatively uniform Grandis-CC and Grandis-RC.
Similarly, Grandis-CC forms an isolated cluster with the highest mean thermal response and lowest variability.
In contrast, Poplar and Grandis-RC show lower mean thermal response, with Poplar notably containing several high-variance outliers.
Finally, testbed temperature analysis indicates that the heat source beneath Poplar is generally cooler and more variable, while it remains higher and more stable beneath both Grandis-CC and Grandis-RC.

\paragraph{Inverse Visual-Thermal Conductance Relationship.}
To evaluate the relationship between visual characteristics and thermal behavior, we conduct a pixel-to-pixel analysis (Figure~\ref{fig:pixel-to-pixel_analysis}) correlating grayscale image intensities with the temperature gradient between the testbed $T_{\text{bed}}$ and the wood surface $T$. 
This thermal difference is quantified as $\Delta T = T_{\text{bed}} - T$ and inverted ($1/\Delta T$) to serve as a proxy for thermal conductance. 
Grouping these pixel-level metrics into grayscale bins effectively reduces localized noise and reveals macroscopic physical trends. 
The binned analysis demonstrates a distinct inverse correlation across all samples.
As grayscale image intensity increases, the mean $1/\Delta T$ value consistently transitions from high to low. 
This indicates that \textit{visually darker regions possess higher thermal conductance, whereas brighter regions act as effective thermal insulators with higher resistance}. 
This high-to-low trend is particularly evident in the Grandis-CC and Grandis-RC datasets, which exhibit steep, consistent downward slopes and narrow standard deviations. 
Poplar also exhibits an inverse relationship, with thermal conductance decreasing slightly but consistently as grayscale image intensity increases.

\begin{figure}[t]
    \centering
    \includegraphics[width=1.0\textwidth]{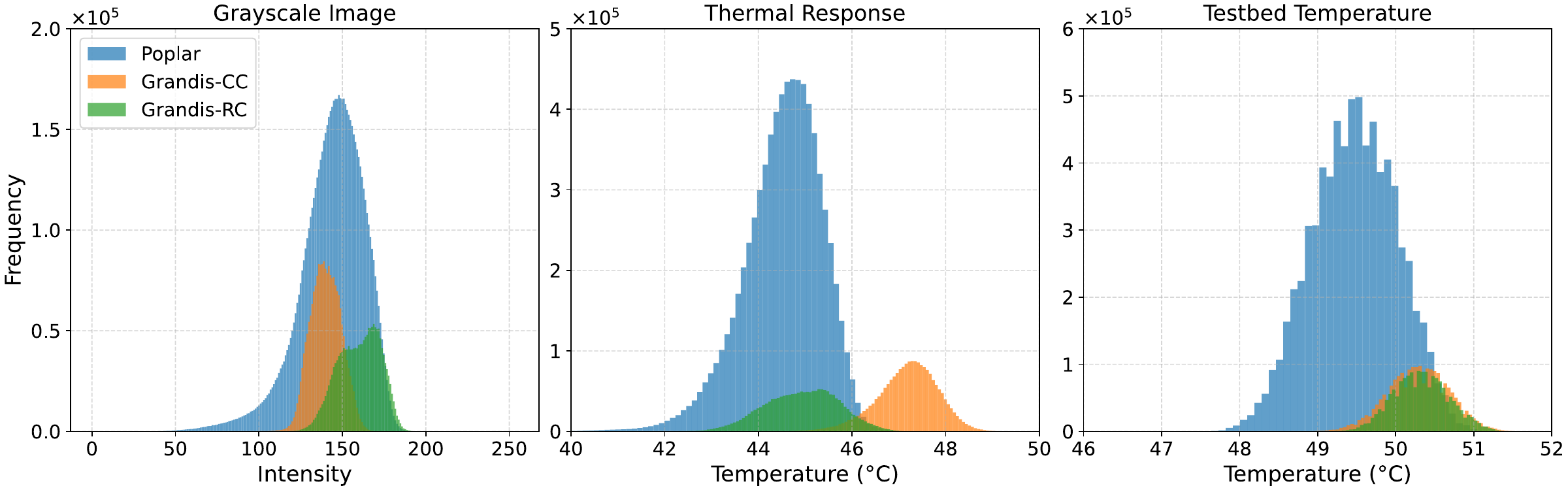}
    \caption{
    Overall sample distributions of grayscale image intensities (left), thermal responses (middle), and testbed temperatures (right) across the Poplar, Grandis-CC, and Grandis-RC datasets.
    }
    \label{fig:overall_histograms}
\end{figure}

\begin{table}[t]
\caption{
     Statistical information of grayscale image intensities, thermal responses, and testbed temperatures across the Poplar, Grandis-CC, and Grandis-RC datasets.
    The results are presented in a $\mu \pm \sigma$ format.
}
\label{tab:compact_variation_stats}
\centering
\resizebox{\textwidth}{!}{
\begin{tabular}{l cc cc cc}
\toprule
\multirow{2}{*}{\textbf{Wood}} & \multicolumn{2}{c}{\textbf{Grayscale Image}} & \multicolumn{2}{c}{\textbf{Thermal Response ($ ^\circ \mathrm{C}$)}} & \multicolumn{2}{c}{\textbf{Testbed Temperature} ($ ^\circ \mathrm{C}$)} \\
\cmidrule(lr){2-3} \cmidrule(lr){4-5} \cmidrule(lr){6-7}
& \textbf{Side A} & \textbf{Side B} & \textbf{Side A} & \textbf{Side B} & \textbf{Side A} & \textbf{Side B} \\
\midrule
Poplar     & $143.82 \pm 19.94$ & $143.83 \pm 19.07$ & $44.65 \pm 0.75$ & $44.39 \pm 0.91$ & $49.43 \pm 0.51$ & $49.51 \pm 0.58$ \\
Grandis-CC & $139.58 \pm 8.90$  & $140.80 \pm 9.40$  & $47.17 \pm 0.64$ & $47.19 \pm 0.66$ & $50.26 \pm 0.40$ & $50.32 \pm 0.39$ \\
Grandis-RC & $160.24 \pm 12.77$ & $160.16 \pm 12.21$ & $44.90 \pm 0.74$ & $44.96 \pm 0.82$ & $50.40 \pm 0.36$ & $50.25 \pm 0.34$ \\
\bottomrule
\end{tabular}
}
\end{table}

\begin{figure}[t]
    \centering
    \includegraphics[width=1.0\textwidth]{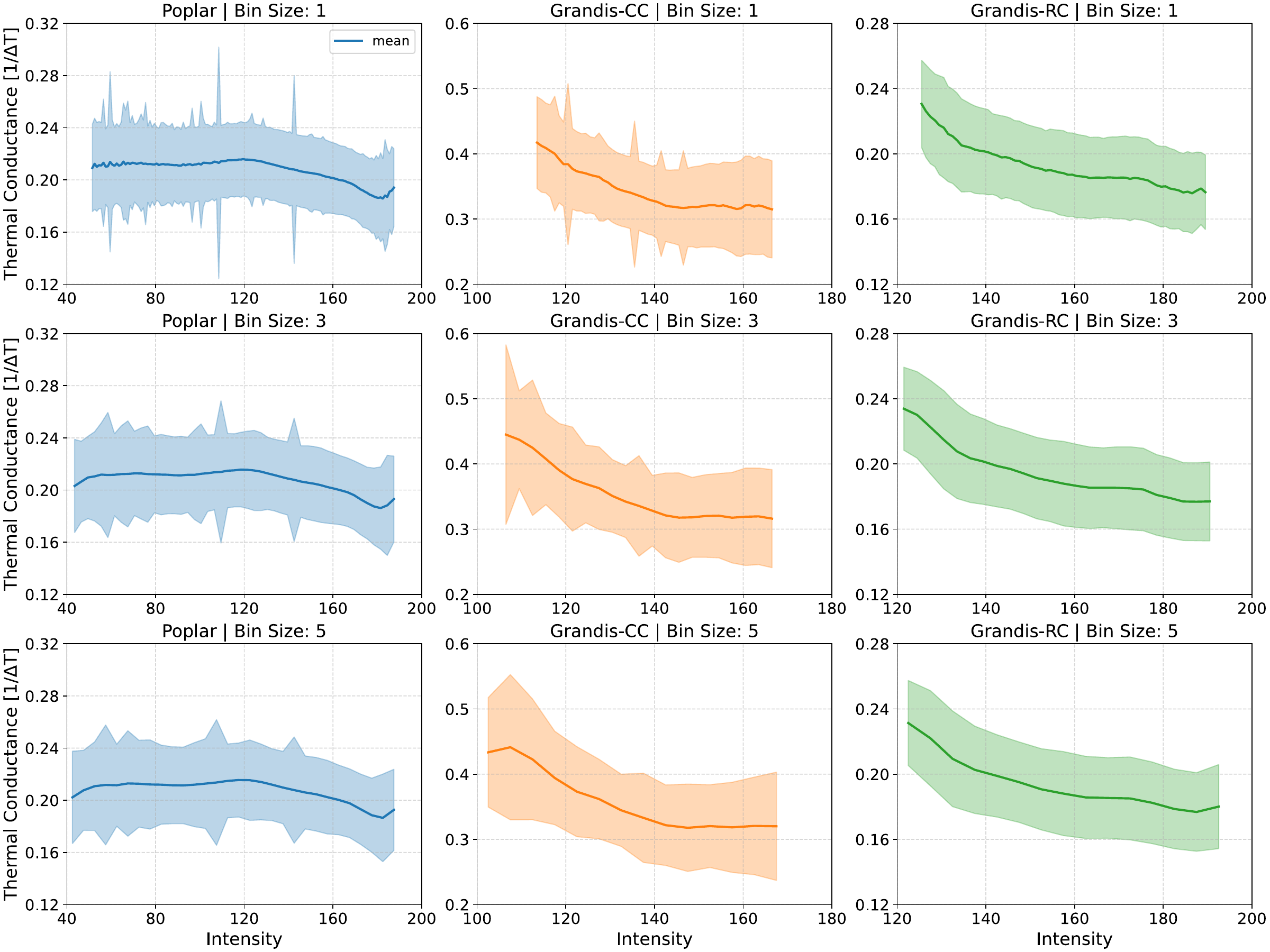}
    \caption{
    Pixel-to-pixel binned distributions of mean and standard deviation values for thermal conductance ($1/\Delta T$) as a function of grayscale image intensities across the Poplar (left), Grandis-CC (middle), and Grandis-RC (right) datasets at varying bin sizes.
    }
    \label{fig:pixel-to-pixel_analysis}
\end{figure}

\paragraph{Similarity in Feature Space.}
While the pixel-level analysis reveals a clear empirical correlation between grayscale intensity and thermal conductance, it ignores the underlying morphology of the wood surface.
To capture the cross-modal morphological alignment between visual appearance and thermal response, we analyze their relationship within the deep feature space of DINOv3~\cite{simeoni2025dinov3}, which is based on small, base, and large Vision Transformers~\cite{dosovitskiy2020image} (ViT-S, ViT-B, ViT-L). 
We extract the output of each transformer block across the network's depth, yielding 12 intermediate feature maps for ViT-S and ViT-B, and 24 for ViT-L. 
Each map is reshaped into a spatial tensor of dimension $c \times h \times w$.
To quantify the cross-modal morphological alignment within these feature maps, each feature map is normalized along the channel dimension $c$, and the cosine similarity is averaged across the $h \times w$ spatial dimension
\begin{equation}
    \text{sim}(I, T) = \frac{1}{h w} \sum_{j=1}^{h} \sum_{i=1}^{w} \text{cos\_sim}\left(\frac{\phi(I)_{i,j}}{\|\phi(I)_{i,j}\|}, \frac{\phi(T)_{i,j}}{\|\phi(T)_{i,j}\|}\right),
\end{equation}
where $\phi(\cdot)_{i,j} \in \mathbb{R}^c$ denotes the $c$-dimensional feature vector at spatial location $(i,j)$.

To track this relationship across network depth, Figure~\ref{fig:similarity} illustrates the mean similarity scores at each layer, calculated across all wood samples for each wood type. 
Across all datasets and model scales, the cross-modal morphological alignment follows a distinct nonlinear trajectory. 
Similarity begins at a moderate baseline in shallow layers, rises to a prominent peak in the intermediate layers (eg, layers 4--5 for ViT-B and 8--10 for ViT-L), and subsequently declines in the deepest layers.
The models exhibit distinct behaviors based on their capacity.
In the smaller ViT-S architecture, similarity increases rapidly and reaches a sustained plateau (layers 4--7), suggesting consistent feature preservation. 
The larger ViT-B and ViT-L models exhibit more localized alignment peaks (around layer 5 for ViT-B and layers 8--10 for ViT-L), followed by a pronounced decline as the network depth increases.
This trajectory also provides critical insight into the nature of the RGB-thermal relationship. 
The deepest layers in high-capacity models like ViT-L drive representations toward modality-divergent abstract semantics. However, the intermediate layers optimally isolate the shared physical structures (eg, grain geometry and morphology) that govern both optical patterns and heat dissipation.
This peak in similarity confirms that the relationship between RGB and thermal responses is fundamentally rooted in the wood's underlying structural morphology.

\begin{figure}[t]
    \centering
    \includegraphics[width=1\textwidth]{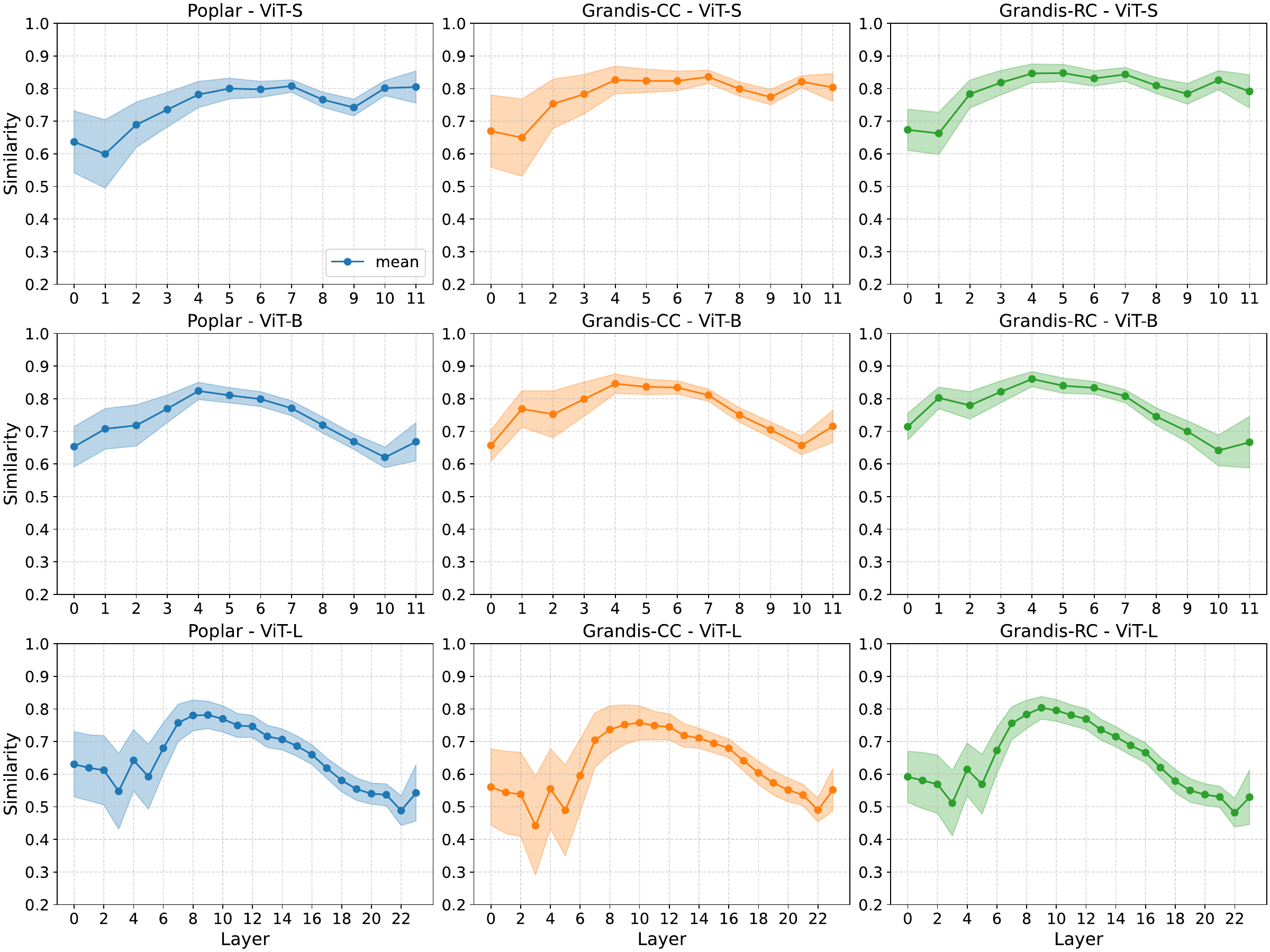}
    \caption{
    Layer-wise RGB image-thermal feature similarity for DINOv3 models~\cite{simeoni2025dinov3} (top to bottom: ViT-S, ViT-B, ViT-L) across the Poplar (left), Grandis-CC (middle), and Grandis-RC (right) datasets.
    }
    \label{fig:similarity}
\end{figure}

\section{Wood Thermal Modeling}
\label{sec:wood_thermal_modeling}

\paragraph{Wood Heat Transfer Path.}
The heat transfer path originates at the aluminum plate and propagates from the bottom to the top surface of the wood sample through heat conduction. 
At the top surface, heat dissipates into the ambient environment via natural convection and thermal radiation. 
Ultimately, the infrared camera captures the surface thermal response via thermal radiation.
Based on this heat transfer path, we can establish a unified mathematical model.

\paragraph{Governing Equation.}
With no internal heat source, heat conduction within a wood sample at the steady state is described as the steady heat conduction equation~\cite{Cengel2006}
\begin{align}\label{eq:3d_anisotropic_general1}
\frac{\partial}{\partial x} \left( k_{x} \frac{\partial T}{\partial x} \right)
+\frac{\partial}{\partial y} \left( k_{y} \frac{\partial T}{\partial y} \right) +
\frac{\partial}{\partial z} \left( k_{z} \frac{\partial T}{\partial z} \right) = 0,
\end{align}
where $T = T(x, y, z)$ is a temperature field and $k_x$, $k_y$, and $k_z$ are thermal conductivities along the $x$, $y$, and $z$ directions, respectively.
Assuming that the thermal conductivities $k_{x}$, $k_{y}$, and $k_{z}$ are uniform over a small neighborhood, Equation~\ref{eq:3d_anisotropic_general1} is reduced to 
\begin{align}
k_{x} \frac{\partial^2 T}{\partial x^2} + \frac{\partial k_{x}}{\partial x} \frac{\partial T}{\partial x} + k_{y} \frac{\partial^2 T}{\partial y^2} + \frac{\partial k_{y}}{\partial y} \frac{\partial T}{\partial y} + k_{z} \frac{\partial^2 T}{\partial z^2} + \frac{\partial k_{z}}{\partial z} \frac{\partial T}{\partial z} &= 0, \\
k_{x} \frac{\partial^2 T}{\partial x^2} +
k_{y} \frac{\partial^2 T}{\partial y^2} +
k_{z} \frac{\partial^2 T}{\partial z^2} &= 0. \label{eq:3d_anisotropic_simple1}
\end{align}

\paragraph{Boundary Conditions.}
Since the bottom surface directly contacts the testbed, the bottom surface satisfies a Dirichlet boundary condition
\begin{equation}
    T(x, y, z=0) = T_{\text{bed}}.
\end{equation}
The wood sample is exposed to the air, and the top surface satisfies a Robin boundary condition
\begin{equation}
\left . - k_{z} \frac{\partial T} {\partial z} \right|_{L_z} = h_{c} \left( T_{\text{top}} - T_{\infty} \right) + \sigma \varepsilon \left(T^{4}_{\text{top}} - T_{\infty}^{4} \right), \label{eq:robin}
\end{equation}
where $h_{c}$ is the convective coefficient.
$T_{\text{top}}$ is the top surface temperature.
$T_{\infty}$ is the ambient temperature.
$\varepsilon$ is the emissivity of the top surface.
$\sigma$ is the Stefan-Boltzmann constant.
To simplify Equation~\ref{eq:robin}, we approximate $\sigma \varepsilon \left(T^{4}_{\text{top}} - T_{\infty}^{4} \right)$ as $4 \sigma \varepsilon T_{\infty}^{3} \left(T_{\text{top}} - T_{\infty} \right)$
\begin{align}
\left . -k_{z}  \frac{\partial T} {\partial z} \right|_{L_z} = h \left( T_{\text{top}} - T_{\infty} \right),
\end{align}
where $h = h_{c} + 4\sigma \varepsilon T_{\infty}^{3}$ is the equivalent convective coefficient.
The other lateral surfaces satisfy Neumann boundary conditions
\begin{align}
\left. \frac{\partial T} {\partial x} \right |_{0}=0, 
\left. \frac{\partial T} {\partial x} \right |_{L_{x}}=0, 
\left. \frac{\partial T} {\partial y} \right |_{0}=0, 
\left. \frac{\partial T} {\partial y} \right |_{L_{y}}=0.
\end{align}

\paragraph{Measurement Formula for Infrared Cameras.}
Since the distance between the wood sample and the infrared camera is short, atmospheric attenuation is negligible. 
The total spectral radiance captured by the camera lens,$L_{\text{cam}}$, can be modeled from two distinct perspectives: the camera's internal calculation and the true physical emission.
First, based on the camera's internal software, it processes $L_{\text{cam}}$ using a predefined emissivity $\bar{\varepsilon}$ to output a measured radiance $L_{\text{mea}}$, while accounting for the reflected radiance $L_{\text{refl}}$
\begin{align}
    \label{eq:actual_radiance}
    L_{\text{cam}} = \bar{\varepsilon} L_{\text{mea}} + (1 - \bar{\varepsilon}) L_{\text{refl}}.
\end{align}
Second, in physical reality, the actual emissivity of the wood $\varepsilon$ differs slightly from the fixed setting $\bar{\varepsilon}$. 
The true radiance emitted by the wood surface is $L_{\text{wood}}$
\begin{align}\label{eq:theoretical_radiance}
    L_{\text{cam}} = \varepsilon L_{\text{wood}} + (1 - \varepsilon) L_{\text{refl}}.
\end{align}
Since the total energy hitting the sensor $L_{\text{cam}}$ must be identical in both scenarios, combining Equation~\ref{eq:actual_radiance} and Equation~\ref{eq:theoretical_radiance} yields 
\begin{align} 
    \bar{\varepsilon} (L_{\text{mea}} - L_{\text{refl}}) = \varepsilon (L_{\text{wood}} - L_{\text{refl}}). \label{eq:meas_formula}
\end{align}
To convert radiance into temperature, we approximate the band-limited radiance $L_{\text{band}}(T)$ using $c T^{4}$, i.e., $L_{\text{band}}(T) = c T^{4}$, where $c$ is a constant.
We define the camera's displayed temperature as $T_{\text{mea}}$, the true top surface temperature as $T_{\text{top}}$, and the ambient reflection temperature as $T_{\text{refl}} = T_{\infty}$.
Substituting $L_{\text{mea}} = c T_{\text{mea}}^{4}$, $L_{\text{wood}} = c T_{\text{top}}^{4}$, and $L_{\text{refl}} = c T_{\infty}^{4}$ into Equation~\ref{eq:meas_formula} yields
\begin{align}
T_{\text{mea}}^{4} - T_{\infty}^{4} &= \frac{\varepsilon}{\bar{\varepsilon}} (T^{4}_{\text{top}} - T_{\infty}^{4}). \label{eq:radiance}
\end{align}
Following the first-order approximation ($T^{4} - T_{\infty}^{4} \approx 4T_{\infty}^{3} (T - T_{\infty}) $),
Equation~\ref{eq:radiance} is reduced to
\begin{align}
    \tilde{T}_{\text{mea}} &= \gamma \tilde{T}_{\text{top}},
    \label{eq:linearization}
\end{align}
where $\gamma = \varepsilon/\bar{\varepsilon}$, $\tilde{T}_{\text{mea}} = T_{\text{mea}} - T_{\infty}$, and $\tilde{T}_{\text{top}} = T_{\text{top}} - T_{\infty}$.

\section{Reparameterizing Governing Equations}
\label{sec:re_pde_app}
To make Equation~\ref{eq:3d_anisotropic_simple1} suitable for our 2D thermal prediction, we eliminate the second derivative with respect to $z$.
We define the average temperature across the thickness of the wood sample as
\begin{align}
    \bar{T}(x,y) = \dfrac{1}{L_{z}} \int_{0}^{L_{z}} T(x, y, z) \mathrm{d} z.
\end{align}
We assume that, for a single pixel, $k_{z}$ is constant along the $z$ direction, integrating Equation~\ref{eq:3d_anisotropic_simple1} over $z \in [0, L_{z}]$ yields
\begin{align}
    \dfrac{1}{L_{z}} \int_{0}^{L_{z}} 
    \left[ k_{x} \frac{\partial^2 T}{\partial x^2} +
    k_{y} \frac{\partial^2 T}{\partial y^2} + k_{z}\frac{\partial^2 T}{\partial z^2} \right] \mathrm{d} z = 0,
\end{align}
and exchanging the derivative and integral for $x$ and $y$ yields
\begin{align}
    k_{x}
    \frac{\partial^2 \bar{T} }{\partial x^2} + k_{y}  \frac{\partial^2 \bar{T}}{\partial y^2}
    &+ \dfrac{k_{z}}{L_{z}} \left[
    \left . \frac{\partial T}{\partial z} \right |_{L_{z}} -
    \left . \frac{\partial T}{\partial z} \right |_{0}
    \right] = 0. \label{eq:integral_z}
\end{align}
For heat flux at $z=0$ and $z=L_{z}$, it yields
\begin{align}
    \left . \dfrac{\partial T}{\partial z} \right |_{L_{z}} \approx - \dfrac{h}{k_{z}} \left( \bar{T} - T_{\infty} \right), \quad 
    \left . \dfrac{\partial T}{\partial z} \right |_{0} \approx \dfrac{\bar{T} - T_{\text{bed}}}{L_{z}}.
\end{align}
Substituting heat flux into Equation~\ref{eq:integral_z} yields
\begin{align}
    \label{eq:simplified_2d}
    k_{x}
    \frac{\partial^2 \bar{T} }{\partial x^2} + k_{y}  \frac{\partial^2 \bar{T}}{\partial y^2}
    - \dfrac{h}{L_{z}} \left( \bar{T} - T_{\infty} \right) -
    \dfrac{k_{z}}{L_{z}^{2}} \left(\bar{T} - T_{\text{bed}} \right)
     &= 0.
\end{align}

\paragraph{Full Governing Equation.}
Since the wood sample is sufficiently thin, we can approximate $\bar{T} \approx T_{\text{top}}$.
Substituting $\tilde{T}_{\text{top}} = T_{\text{top}} - T_{\infty}$, $\tilde{T}_{\text{bed}} = T_{\text{bed}} - T_{\infty}$, and Equation~\ref{eq:linearization} into Equation~\ref{eq:simplified_2d} yields
\begin{align}
    \alpha L_{z}^{2} \frac{\partial^2  \tilde{T}_{\text{mea}} }{\partial x^2} &+ \beta L_{z}^{2} \frac{\partial^2  \tilde{T}_{\text{mea}}}{\partial y^2} - (1 + \text{Bi})\tilde{T}_{\text{mea}} + \gamma \tilde{T}_{\text{bed}} = 0, \label{eq:pde_2d1}
\end{align}
where $\alpha = k_{x} / k_{z}$, $\beta = k_{y}/k_{z}$, and $\text{Bi} = h L_{z} / k_{z}$ is a Biot number.
The boundary conditions at the top and bottom surfaces are included in Equation~\ref{eq:pde_2d1}, and the boundary conditions at the other lateral surfaces maintain
\begin{align}
\left. \frac{\partial \tilde{T}_{\text{mea}}} {\partial x} \right |_{0}=0, 
\left. \frac{\partial \tilde{T}_{\text{mea}}} {\partial x} \right |_{L_{x}}=0,
\left. \frac{\partial \tilde{T}_{\text{mea}}} {\partial y} \right |_{0}=0, 
\left. \frac{\partial \tilde{T}_{\text{mea}}} {\partial y} \right |_{L_{y}}=0.
\end{align}
Based on Equation~\ref{eq:pde_2d1}, the parameters $\alpha$, $\beta$, $\gamma$, and $\text{Bi}$ have distinct physical meanings and their theoretical physical bounds
\begin{align}
    \alpha > 0, \quad \beta > 0, \quad  \text{Bi} > 0, \quad \gamma > 0.
\end{align}
The loose upper bound for $\gamma$ accommodates physically plausible variations in actual wood emissivity relative to standard default camera settings (typically $\bar{\varepsilon} \approx 0.95$). 
These parameters are learned from RGB images only, which can be parameterized as $\alpha = \alpha(I)$, $\beta = \beta(I)$, $\gamma = \gamma(I)$, and $\text{Bi} = \text{Bi}(I)$.

\paragraph{Reduced Governing Equation.}
To reduce the number of learnable parameters, we can further rewrite Equation~\ref{eq:pde_2d1} as
\begin{align}
    \label{eq:pde_2d_simple}
    \tilde{T}_{\text{mea}} &=  r_{\gamma} \tilde{T}_{\text{bed}} + r_{\alpha} L_{z}^{2} \frac{\partial^2 \tilde{T}_{\text{mea}} }{\partial x^2} + r_{\beta} L_{z}^{2} \frac{\partial^2 \tilde{T}_{\text{mea}}}{\partial y^2},
\end{align}
and learn the ratios $r_{\alpha}$, $r_{\beta}$, and $r_{\gamma}$ with physical bounds
\begin{align}
    r_{\alpha} = \frac{\alpha}{1 + \text{Bi}} > 0,\quad  r_{\beta} = \frac{\beta }{1 + \text{Bi}} > 0, \quad r_{\gamma} = \frac{\gamma}{1 + \text{Bi}} > 0.
\end{align}

\paragraph{Normalized Governing Equation.}
To account for individual wood sample variations, each thermal response or testbed temperature map is standardized by its own spatial mean and standard deviation, defined as $\tilde{T}^* = (\tilde{T} - \mu) / \sigma$.
This transformation is primarily motivated by the need to decouple intrinsic material properties from extrinsic experimental variables.
Substituting $\tilde{T}_{\text{mea}} = \sigma_{\text{mea}} \tilde{T}^*_{\text{mea}} + \mu_{\text{mea}}$ and $\tilde{T}_{\text{bed}} = \sigma_{\text{bed}} \tilde{T}^*_{\text{bed}} + \mu_{\text{bed}}$ into Equation~\ref{eq:pde_2d_simple} yields
\begin{align}
    \label{eq:norm_pde_2d_discrete1}
    \tilde{T}^*_{\text{mea}} = \underbrace{\left( \frac{ r_{\gamma} \sigma_{\text{bed}}}{\sigma_{\text{mea}}} \right)}_{k^*} \tilde{T}^*_{\text{bed}} + \underbrace{\left( \frac{r_{\gamma} \mu_{\text{bed}} - \mu_{\text{mea}}}{\sigma_{\text{mea}}} \right)}_{b^*} + \underbrace{r_{\alpha} L_{z}^{2} \frac{\partial^2 \tilde{T}^*_{\text{mea}}}{\partial x^2} + r_{\beta} L_{z}^{2} \frac{\partial^2 \tilde{T}^*_{\text{mea}}}{\partial y^2} }_{\text{Spatial Diffusion}}.
\end{align}

\section{Physics-Informed Neural Network Architecture}
\label{sec:architectures}
Since our datasets are small, we use a U-Net architecture~\cite{ronneberger2015u} as a backbone to regress thermal responses.
The U-net extracts multi-scale features with channel dimensions doubling at each downsampling step, i.e., $C_{5} = 2 C_4 = 4 C_3 = 8 C_2 = 16 C_1$.
To maintain computational efficiency, the base channel count $C_{1}$ is set to $16$.
Each block consists of a single convolution, followed by batch normalization (BN)~\cite{Ioffe2015batch} and ReLU activation~\cite{nair2010rectified}, i.e., [Conv $\to$ BN $\to$ ReLU].

\paragraph{Physics-Informed Convolutional Neural Networks (PICNN).} 
The physical parameters ($\hat{r}_{\alpha}$, $\hat{r}_{\beta}$, $k^{*}$, $b^{*}$) are learned exclusively from wood RGB images (Figure~\ref{fig:framework_overview}). 
Instead of concatenating a wood RGB image and a testbed temperature map, we split the final feature map $z_{\text{out}} \in \mathbb{R}^{C \times H \times W}$ output by the U-net into two parts.
One is for learning the physical parameters, and the other is for regressing the thermal response by incorporating the testbed temperature map, $T_{\text{bed}}$, via feature-wise linear modulation (FiLM)~\cite{perez2018film}. 
This ensures the Dirichlet boundary condition is integrated efficiently, linearly modulating the features as $\tilde{z}_{\text{out}} = z_{\text{out}} \odot a(T_{\text{bed}}) + b(T_{\text{bed}})$.

\paragraph{Physics-Integrated Convolutional Neural Networks (PInteCNN).} 
For the explicitly constrained neural network, the U-Net~\cite{ronneberger2015u} relies entirely on wood RGB images to learn the physical parameters $( \theta_{\alpha}, \theta_{\beta}, \theta_{\text{bed}}, \theta_{s})$, as illustrated in Figure~\ref{fig:framework_overview}. 
These physical parameters are then directly routed through the analytical approximator-predictor-corrector modules to compute the final thermal response predictions.

\section{Mathematical Proof of Bivariate Partitioning Unbiasedness}
\label{sec:proof_bivariate}

\begin{algorithm}[t]
\caption{KS-Optimized Bivariate Data Partitioning}
\label{alg:bivariate_split}
\begin{algorithmic}[1]
\Require Wood samples $\mathcal{S}$, total subsets $M$, split sizes $M_{\text{train}}, M_{\text{val}}, M_{\text{test}}$, and threshold $\tau$
\Ensure Partitioned sets $\mathcal{D}_{\text{train}}, \mathcal{D}_{\text{val}}$, and $\mathcal{D}_{\text{test}}$

\State \textbf{1. Metric Extraction and Filtering}
\State Initialize valid set $\mathcal{V} \leftarrow \emptyset$
\For{each sample $s_n \in \mathcal{S}$}
    \State Extract $T^*_{\text{mea}}$ and $T^*_{\text{bed}}$, and regress to find $k^*_A$ and $ k^*_B$
    \If{$k^*_A \ge \tau$ \textbf{and} $k^*_B \ge \tau$}
        \State Compute mean $\overline{k}^{*}_{n} \leftarrow \frac{1}{2}(k^*_A + k^*_B)$ and asymmetry $\Delta k^{*}_{n} \leftarrow |k^*_A - k^*_B|$
        \State $\mathcal{V} \leftarrow \mathcal{V} \cup \{ s_n \}$
    \EndIf
\EndFor

\State \textbf{2. Bivariate Sorting and Shifted Round-Robin Allocation}
\State Sort $\mathcal{V}$ descending by $\overline{k}^{*}$, and divide into $P$ contiguous batches of size $M$
\State Initialize $M$ empty subsets $\mathcal{D}_1, \dots, \mathcal{D}_M \leftarrow \emptyset$
\For{each batch index $p \in \{1, \dots, P\}$}
    \State Sort the $M$ samples in batch $p$ descending by $\Delta k^{*}$
    \For{each sample rank $j \in \{1, \dots, M\}$ in batch $p$}
        \State Assign sample to subset $\mathcal{D}_{m}$, where $m \leftarrow ((p + j - 2) \bmod M) + 1$
    \EndFor
\EndFor

\State \textbf{3. KS-Optimized Split Construction}
\State Initialize minimal distribution distance $D_{\min} \leftarrow \infty$
\For{each valid partition of $\mathcal{D}_1, \dots, \mathcal{D}_M$ into sets $C_{\text{train}}, C_{\text{val}}, C_{\text{test}}$}
    \State Compute two-sample KS statistics for $\overline{k}^*$ and $\Delta k^*$ across the candidate splits
    \State Aggregate KS statistics into a total distance score $D_{\text{KS}}$
    \If{$D_{\text{KS}} < D_{\min}$}
        \State $D_{\min} \leftarrow D_{\text{KS}}$
        \State $\mathcal{D}_{\text{train}}, \mathcal{D}_{\text{val}}, \mathcal{D}_{\text{test}} \leftarrow C_{\text{train}}, C_{\text{val}}, C_{\text{test}}$
    \EndIf
\EndFor

\State \Return $\mathcal{D}_{\text{train}}, \mathcal{D}_{\text{val}}, \mathcal{D}_{\text{test}}$
\end{algorithmic}
\end{algorithm}

To prove the unbiasedness of this algorithm, let the valid dataset $\mathcal{V}$ of size $L$ be divided into batches of size $M$. In the most generalized real-world scenario, $L$ is not perfectly divisible by $M$. Thus, $L = P \times M + R_1$, yielding $P$ complete batches and one final incomplete batch of size $R_1$ ($0 \le R_1 < M$). Furthermore, the $P$ complete batches form $C$ complete cycles of length $M$, plus a remainder of incomplete cycles $R_2$, such that $P = C \times M + R_2$ ($0 \le R_2 < M$). Algorithm~\ref{alg:bivariate_split} distributes these samples into $M$ disjoint subsets $\mathcal{D}_1, \dots, \mathcal{D}_M$. The subsets assigned samples from the $R_1$ remainder batch have a cardinality of $L_m = |\mathcal{D}_m| = P + 1$, while the rest have $L_m = |\mathcal{D}_m| = P$.

\begin{theorem}
\label{thm:primary_metric}
Algorithm~\ref{alg:bivariate_split} guarantees that the sample mean and sample variance of $\overline{k}^*$ for every subset $\mathcal{D}_m$ converge to the global dataset mean and global dataset variance, respectively, at a rate of $\mathcal{O}(1/P)$.
\end{theorem}

\begin{proof}
Because $\mathcal{V}$ is sorted descending by $\overline{k}^*$ before batching, the local extrema of adjacent batches naturally overlap, satisfying $\overline{k}^*_{p, \text{min}} \ge \overline{k}^*_{p+1, \text{max}}$. Thus, for all internal batch transitions $p \in \{1, \dots, P-1\}$, the difference is non-positive ($\overline{k}^*_{p+1, \text{max}} - \overline{k}^*_{p, \text{min}} \le 0$).

To prove mean convergence, we evaluate the absolute deviation of the subset mean $\mu_m$ from the global mean $\mu$. Because $\mu$ is the weighted average of all subset means, this deviation is strictly bounded by the maximum absolute difference between any two subsets ($|\mu_m - \mu| \le \max |\mu_m - \mu_n|$). To evaluate this inter-subset difference, let $S_m$ and $S_n$ denote their respective subset sums. For a dataset geometry $L = P \times M + R_1$, the maximum cardinality mismatch is $L_m = P+1$ and $L_n = P$. Expanding the mean difference yields
\begin{equation}
|\mu_m - \mu_n| = \left| \frac{S_m}{P+1} - \frac{S_n}{P} \right| \le \frac{|S_m - S_n|}{P+1} + \frac{\mu_n}{P+1}.
\end{equation}
To bound the sum disparity $|S_m - S_n|$, we segment it into the $P$ complete batches and the potential $R_1$ remainder batch:
\begin{align}
|S_m - S_n| &\le \sum_{p=1}^P \left( \overline{k}^*_{p, \text{max}} - \overline{k}^*_{p, \text{min}} \right) + \overline{k}^*_{P+1, \text{max}} \notag \\
&= \overline{k}^*_{1, \text{max}} + \sum_{p=1}^{P-1} \left( \overline{k}^*_{p+1, \text{max}} - \overline{k}^*_{p, \text{min}} \right) - \overline{k}^*_{P, \text{min}} + \overline{k}^*_{P+1, \text{max}}.
\end{align}
Because every internal telescoping term is $\le 0$, the maximum sum difference isolates strictly to the global boundaries. Let this finite constant be denoted as $\Delta S_{\max} = \overline{k}^*_{1, \text{max}} - \overline{k}^*_{P, \text{min}} + \overline{k}^*_{P+1, \text{max}}$. Substituting this back into the mean deviation and applying $\frac{1}{P+1} < \frac{1}{P}$, we write
\begin{equation}
\max |\mu_m - \mu_n| \le \frac{\Delta S_{\max} + \overline{k}^*_{\max}}{P}.
\end{equation}
Because the bounding numerator is a finite constant completely independent of $P$, we define it as $C_{\text{mean}}$. Thus, we establish $|\mu_m - \mu| \le \frac{C_{\text{mean}}}{P}$, directly satisfying the formal definition of deterministic convergence at a rate of $\mathcal{O}(1/P)$.

To prove variance convergence, we evaluate the absolute deviation $|\sigma_m^2 - \sigma^2|$. We define the subset variance using the variance shift identity relative to the global mean $\mu$:
\begin{equation}
\sigma_m^2 = \left[ \frac{1}{L_m} \sum_{p=1}^{L_m} (\overline{k}^*_{p,m} - \mu)^2 \right] - (\mu_m - \mu)^2.
\end{equation}
By subtracting the exact global variance $\sigma^2$ from both sides and applying the absolute value, we define the explicit equation for the variance deviation:
\begin{equation}
\left| \sigma_m^2 - \sigma^2 \right| = \left| \underbrace{\left[ \frac{1}{L_m} \sum_{p=1}^{L_m} (\overline{k}^*_{p,m} - \mu)^2 \right]}_{\text{Pseudo-Variance}} - \sigma^2 - (\mu_m - \mu)^2 \right|.
\end{equation}
Applying the triangle inequality separates the pseudo-variance error from the mean-shift penalty:
\begin{equation}
\left| \sigma_m^2 - \sigma^2 \right| \le \left| \frac{1}{L_m} \sum_{p=1}^{L_m} (\overline{k}^*_{p,m} - \mu)^2 - \sigma^2 \right| + (\mu_m - \mu)^2.
\end{equation}
The dataset's maximum absolute gradient is strictly bounded by the Lipschitz constant $Q = 2 \max(\overline{k}^*_{\max} - \mu, \mu - \overline{k}^*_{\min})$. Utilizing $Q$, the maximum difference in the sum of squared deviations is explicitly bounded by $Q \cdot \Delta S_{\max}$. Thus, the pseudo-variance error isolates to $\frac{Q \cdot \Delta S_{\max}}{L_m}$, yielding
\begin{equation}
\left| \sigma_m^2 - \sigma^2 \right| \le \frac{Q \cdot \Delta S_{\max}}{L_m} + (\mu_m - \mu)^2.
\end{equation}
Noting that $\frac{1}{L_m} \le \frac{1}{P}$ and substituting the previously established mean bound $C_{\text{mean}}$ yields
\begin{equation}
\left| \sigma_m^2 - \sigma^2 \right| \le \frac{Q \cdot \Delta S_{\max}}{P} + \left( \frac{C_{\text{mean}}}{P} \right)^2.
\end{equation}
Because $Q$, $\Delta S_{\max}$, and $C_{\text{mean}}$ are finite constants, the $\mathcal{O}(1/P^2)$ penalty term rapidly decays. The dominant linear bound establishes $|\sigma_m^2 - \sigma^2| \le \frac{C_{\text{var}}}{P}$. Consequently, the subset variance uniformly reconstructs the global dataset variance, mathematically guaranteeing convergence at a strict rate of $\mathcal{O}(1/P)$.
\end{proof}

\begin{theorem}
\label{thm:secondary_metric}
Algorithm~\ref{alg:bivariate_split} guarantees that the sample mean and sample variance of $\Delta k^*$ for every subset $\mathcal{D}_m$ converge to the global dataset mean and global dataset variance, respectively, at a rate of $\mathcal{O}(1/P)$.
\end{theorem}

\begin{proof}
Within any batch $p$, the local distribution is linearly approximated by gradient $\delta_p$ and mean $\mu_p$. By inverting the cyclical assignment rule, the exact rank assigned to subset $m$ from batch $p$ is $j(p, m) = ((m - p) \bmod M) + 1$. The assigned sample is thus linearly approximated as $\Delta k^*_{p,m} \approx \mu_p + \delta_p \left( \frac{M+1}{2} - j(p, m) \right)$. 

To prove mean convergence, we evaluate the absolute deviation of the subset mean $\mu_m$ from the global mean $\mu$. By substituting the linear approximation, the deviation is driven entirely by the cyclical rank bias over the subset's $L_m$ elements:
\begin{equation}
|\mu_m - \mu| = \frac{1}{L_m} \left| \sum_{p=1}^{L_m} \delta_p \left( \frac{M+1}{2} - j(p, m) \right) \right|.
\end{equation}
The dataset is geometrically structured as $P$ complete batches and a partial $R_1$ batch ($L = P \times M + R_1$), where the $P$ batches form $C$ complete cycles and $R_2$ remainder batches ($P = C \times M + R_2$). Segmenting the sum explicitly isolates the complete cycles and the remainder batch residuals ($\varepsilon_{R_2}$ and $\varepsilon_{R_1}$):
\begin{align} 
|\mu_m - \mu| \approx \frac{1}{L_m} \Bigg| &\underbrace{\sum_{c=1}^C \delta_c \left[ \sum_{p=(c-1)M+1}^{cM} \left( \frac{M+1}{2} - j(p, m) \right) \right]}_{\text{Complete Cycles}} \notag\\
&+ \underbrace{\sum_{p=CM+1}^P \delta_p \left( \frac{M+1}{2} - j(p, m) \right)}_{\varepsilon_{R_2}} + \underbrace{\sum_{p=P+1}^{L_m} \delta_p \left( \frac{M+1}{2} - j(p, m) \right)}_{\varepsilon_{R_1}} \Bigg|. 
\end{align}
Assuming local gradient stationarity ($\delta_p \approx \delta_c$) within a cycle, the assigned ranks $j(p,m)$ perfectly permute $\{1, \dots, M\}$. Thus, the inner sum for the complete cycles strictly evaluates to zero ($\frac{M(M+1)}{2} - \frac{M(M+1)}{2} = 0$). 

The absolute deviation isolates purely to the residual errors from the incomplete remainder batches ($R_2 + 1$ maximum batches). Because the maximum rank deviation is $\frac{M-1}{2}$, we can bound the total residual using the global maximum gradient $\delta_{\max}$. Noting that subset cardinality $L_m \in \{P, P+1\}$ implies $\frac{1}{L_m} \le \frac{1}{P}$, we write
\begin{equation}
|\mu_m - \mu| \le \frac{|\varepsilon_{R_2} + \varepsilon_{R_1}|}{L_m} \le \frac{(R_2 + 1) \cdot \delta_{\max} \left( \frac{M-1}{2} \right)}{P}.
\end{equation}
Because the numerator is a finite constant completely independent of $P$, we establish $|\mu_m - \mu| \le \frac{C_{\text{mean}}}{P}$, proving deterministic convergence at a rate of $\mathcal{O}(1/P)$. 

To prove variance convergence, we similarly evaluate the absolute deviation of the subset variance $\sigma_m^2$ from the global variance $\sigma^2$. The subset variance is $\sigma_m^2 = \frac{1}{L_m} \sum_{p=1}^{L_m} (\Delta k^*_{p,m} - \mu_m)^2$. Substituting the linear approximation expands this into three components:
\begin{align}
\sigma_m^2 \approx \underbrace{\frac{1}{L_m} \sum_{p=1}^{L_m} (\mu_p - \mu_m)^2}_{\text{Macro Variance}} &+ \underbrace{\frac{1}{L_m} \sum_{p=1}^{L_m} \delta_p^2 \left( \frac{M+1}{2} - j(p, m) \right)^2}_{\text{Micro Variance}} \notag\\
&+ \underbrace{\frac{2}{L_m} \sum_{p=1}^{L_m} (\mu_p - \mu_m) \delta_p \left( \frac{M+1}{2} - j(p, m) \right)}_{\text{Cross-Covariance}}.
\end{align}
For the exact global variance $\sigma^2$, cross-covariance is identically zero. For the subset variance, the cross-covariance evaluated over the $C$ complete cycles also collapses to zero due to the perfect rank permutation, and the micro variance evaluates to an identical constant for all $m$. 

Segmenting the summation into the complete cycles and the remainder batches explicitly isolates the residual variance $\varepsilon_{\text{var}}$:
\begin{align}
\sigma_m^2 \approx \sigma^2 + \frac{1}{L_m} \Bigg[ &\underbrace{\sum_{p=CM+1}^{L_m} (\mu_p - \mu_m)^2}_{\varepsilon_{\text{macro}}} + \underbrace{\sum_{p=CM+1}^{L_m} \delta_p^2 \left( \frac{M+1}{2} - j(p, m) \right)^2}_{\varepsilon_{\text{micro}}} \notag\\
&+ \underbrace{\sum_{p=CM+1}^{L_m} 2(\mu_p - \mu_m) \delta_p \left( \frac{M+1}{2} - j(p, m) \right)}_{\varepsilon_{\text{cross}}} \Bigg].
\end{align}
Thus, the absolute difference $|\sigma_m^2 - \sigma^2|$ eliminates the cycle components and isolates strictly to the combined residual error $\varepsilon_{\text{var}} = \varepsilon_{\text{macro}} + \varepsilon_{\text{micro}} + \varepsilon_{\text{cross}}$. Let $\Delta_{\max}^2$ denote the global maximum absolute squared deviation from the mean within the dataset. Applying the remainder batch bound ($R_2 + 1$ maximum batches) and the $\frac{1}{L_m} \le \frac{1}{P}$ inequality yields
\begin{equation}
|\sigma^2_m - \sigma^2| = \frac{|\varepsilon_{\text{var}}|}{L_m} \le \frac{(R_2 + 1) \cdot \Delta_{\max}^2}{P}.
\end{equation}
Because the bounding numerator is a finite constant, we establish $|\sigma^2_m - \sigma^2| \le \frac{C_{\text{var}}}{P}$. Consequently, the subset variance uniformly reconstructs the Law of Total Variance, guaranteeing that the maximum variance deviation decays at a strict rate of $\mathcal{O}(1/P)$.
\end{proof}

\begin{figure}[t]
    \centering
    \includegraphics[width=1\textwidth]{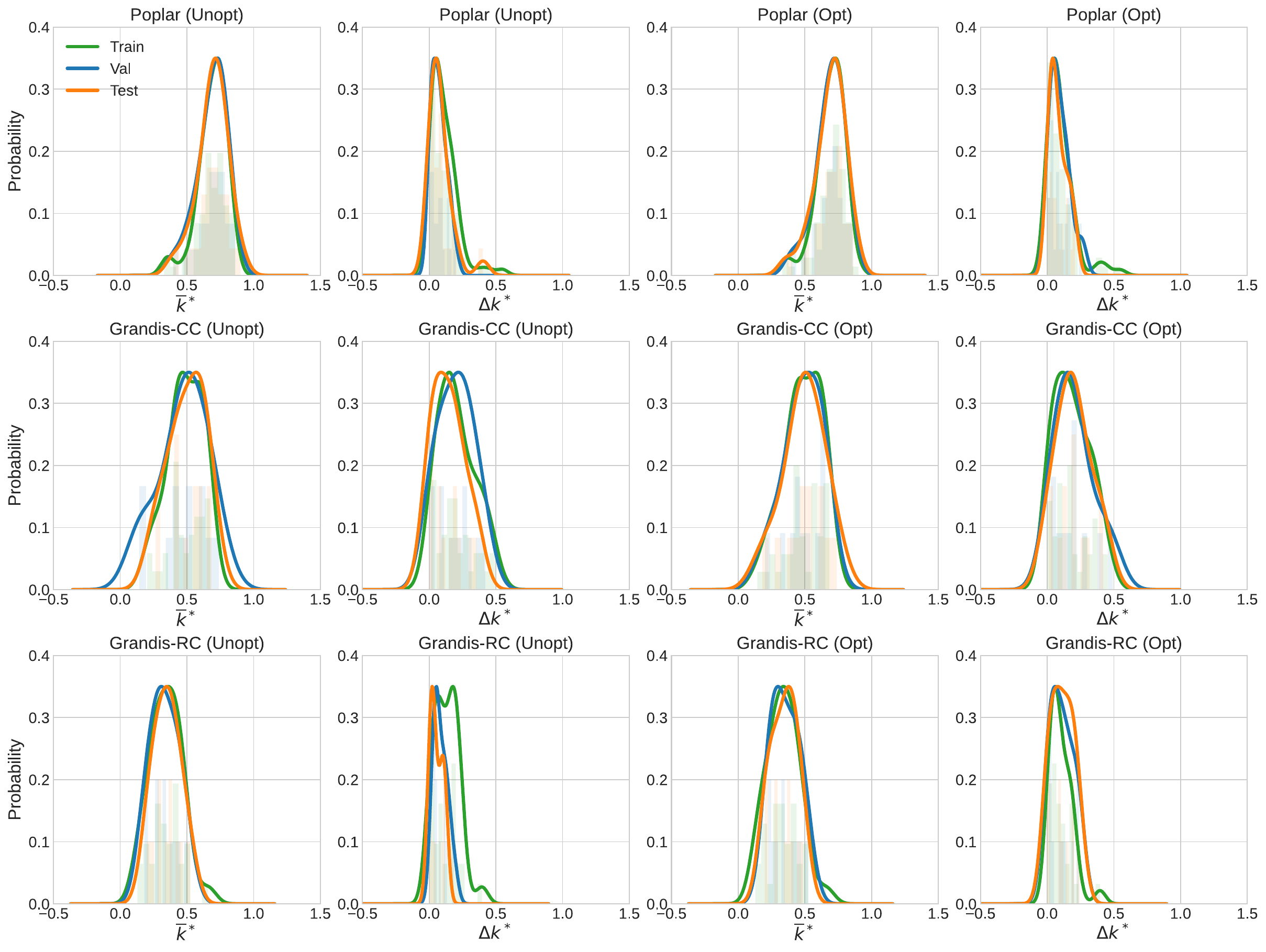}
    \caption{
    Probability distributions of physical metrics ($\overline{k}^*$ and $\Delta k^*$) for training, validation, and test splits across the Poplar (top), Grandis-CC (middle), and Grandis-RC (bottom) datasets.
    While the unoptimized baseline (left two columns) exhibits noticeable shifts in distribution and peak misalignments between subsets, our KS-optimized partitioning (right two columns) aligns the profiles across all splits, effectively mitigating domain shift.
    }
    \label{fig:coefficient_split}
\end{figure}

\section{More Results on Ablation Studies}
\label{sec:ablation}
This section conducts ablation studies on the weighting coefficient $\lambda_{\text{pde}}$ and corrector steps (Table~\ref{tab:model_evaluation}) and learned physical parameters (Table~\ref{tab:merged_picnn_pintecnn}).
\begin{table}[htbp]
  \caption{
  Ablation study of the three thermal prediction approaches. 
  The data-driven approach, PICNN, and PInteCNN are evaluated across the three datasets (Poplar, Grandis-CC, and Grandis-RC). 
  Performance is assessed under varying model configurations (PDE penalty weights $\lambda_{\text{pde}}$ and corrector steps) using MAE, RMSE, and $\delta_{01}$. 
  Bold values denote the best overall configuration for each dataset.
  }
  \label{tab:model_evaluation}
  \centering
  \resizebox{\textwidth}{!}{ 
  \begin{tabular}{ll ccc ccc ccc}
    \toprule
    \multirow{2}{*}{Method} & \multirow{2}{*}{Configuration} & \multicolumn{3}{c}{\textbf{Poplar}} & \multicolumn{3}{c}{\textbf{Grandis-CC}} & \multicolumn{3}{c}{\textbf{Grandis-RC}} \\
    \cmidrule(lr){3-5} \cmidrule(lr){6-8} \cmidrule(lr){9-11}
    & & MAE $\downarrow$ & RMSE $\downarrow$ & $\delta_{01}$ (\%) $\uparrow$ & MAE $\downarrow$ & RMSE $\downarrow$ & $\delta_{01}$ (\%) $\uparrow$ & MAE $\downarrow$ & RMSE $\downarrow$ & $\delta_{01}$ (\%) $\uparrow$ \\
    \midrule
    
    Data-driven & -- & 0.3116 & 0.4023 & 75.64 & 0.2051 & 0.2631 & 92.05 & 0.3578 & 0.4526 & 69.00 \\
    
    \midrule
    \multirow{7}{*}{PICNN} 
    & $\lambda_{\text{pde}} = 10^{0}$  & 0.3223 & 0.4144 & 74.15 & 0.2263 & 0.2941 & 88.89 & 0.3861 & 0.4901 & 65.99 \\
    & $\lambda_{\text{pde}} = 10^{-1}$ & 0.3150 & 0.4049 & 75.17 & \textbf{0.2043} & \textbf{0.2619} & \textbf{92.13} & 0.3503 & 0.4427 & 70.35 \\
    & $\lambda_{\text{pde}} = 10^{-2}$ & 0.3140 & 0.4037 & 75.42 & 0.2088 & 0.2673 & 91.63 & 0.3535 & 0.4466 & 70.22 \\
    & $\lambda_{\text{pde}} = 10^{-3}$ & 0.3090 & 0.4000 & \textbf{76.06} & 0.2051 & 0.2654 & 91.97 & 0.3536 & 0.4462 & 70.06 \\
    & $\lambda_{\text{pde}} = 10^{-4}$ & 0.3121 & 0.4035 & 75.79 & 0.2076 & 0.2683 & 91.66 & 0.3468 & 0.4409 & 71.48 \\
    & $\lambda_{\text{pde}} = 10^{-5}$ & \textbf{0.3088} & \textbf{0.3974} & 76.05 & 0.2065 & 0.2668 & 91.68 & 0.3484 & 0.4416 & 71.13 \\
    & $\lambda_{\text{pde}} = 10^{-6}$ & 0.3125 & 0.4031 & 75.59 & 0.2067 & 0.2674 & 91.65 & \textbf{0.3456} & \textbf{0.4386} & \textbf{71.57} \\
    
    \midrule
    \multirow{3}{*}{PInteCNN} 
    & Steps $= 0$ & 0.3137 & 0.4056 & 75.34 & 0.2163 & 0.2781 & 90.32 & 0.3758 & 0.4772 & 67.38 \\
    & Steps $= 1$ & 0.3096 & 0.3992 & 75.82 & 0.2102 & 0.2694 & 91.35 & 0.3497 & 0.4387 & 70.17 \\
    & Steps $= 2$ & 0.3132 & 0.4026 & 75.38 & 0.2824 & 0.3629 & 81.98 & 0.3620 & 0.4537 & 68.27 \\
    
    \bottomrule
  \end{tabular}
  } 
\end{table}

\begin{table}[htbp]
  \caption{
  Learned physical parameters for physics-informed learning. 
  Part I reports the converged values for ($\hat{r}_{\alpha}$, $\hat{r}_{\beta}$, $k^{*}$, $b^{*}$) across the training, validation, and test sets for the PICNN under varying $\lambda_{\text{pde}}$ weights. 
  Part II details the spatially physical parameters ($\theta_{\alpha}$, $\theta_{\beta}$, $\theta_{\text{bed}}$, $\theta_{\text{s}}$) for the PInteCNN across different corrector step iterations.
  All the physical parameters are computed as spatial averages.
  Bold formatting highlights the parameter rows corresponding to the configuration that achieves the lowest MAE for each dataset.
  }
  \label{tab:merged_picnn_pintecnn}
  \centering
  
  \resizebox{\textwidth}{!}{%
  \begin{tabular}{clcccccccccccc}
    \toprule
    \multicolumn{14}{c}{\textbf{Part I: PICNN Learned Physical Parameters}} \\
    \midrule
    \multirow{2}{*}{\textbf{Dataset}} & \multirow{2}{*}{\textbf{$\lambda_{\text{pde}}$}} & \multicolumn{3}{c}{\textbf{$\hat{r}_{\alpha}$}} & \multicolumn{3}{c}{\textbf{$\hat{r}_{\beta}$}} & \multicolumn{3}{c}{\textbf{$k^{*}$}} & \multicolumn{3}{c}{\textbf{$b^{*}$}} \\
    \cmidrule(lr){3-5} \cmidrule(lr){6-8} \cmidrule(lr){9-11} \cmidrule(lr){12-14}
    & & Train & Val & Test & Train & Val & Test & Train & Val & Test & Train & Val & Test \\
    \midrule
    \multirow{7}{*}{\textbf{Poplar}} 
    & $10^{0}$       & 0.3772 & 0.3798 & 0.3869 & 0.5843 & 0.5964 & 0.6024 & 0.5472 & 0.5553 & 0.5555 & 0.0076 & 0.0073 & 0.0082 \\
    & $10^{-1}$       & 0.1516 & 0.1568 & 0.1633 & 0.3296 & 0.3755 & 0.3813 & 0.6422 & 0.6469 & 0.6441 & 0.0505 & 0.0154 & 0.0300 \\
    & $10^{-2}$      & 0.0947 & 0.1131 & 0.1175 & 0.1102 & 0.1332 & 0.1391 & 0.6241 & 0.6445 & 0.6457 & 0.0451 & 0.0249 & 0.0198 \\
    & $10^{-3}$    & 0.0673 & 0.0753 & 0.0771 & 0.0727 & 0.0815 & 0.0837 & 0.6536 & 0.6215 & 0.6196 & 0.0539 & 0.0813 & 0.0985 \\
    & $10^{-4}$    & 0.0854 & 0.0957 & 0.1013 & 0.0862 & 0.0965 & 0.1022 & 0.6228 & 0.6549 & 0.6516 & 0.0548 & 0.0589 & 0.0765 \\
    & $10^{-5}$ & \textbf{0.1240} & \textbf{0.1400} & \textbf{0.1454} & \textbf{0.1389} & \textbf{0.1548} & \textbf{0.1605} & \textbf{0.7014} & \textbf{0.6419} & \textbf{0.6431} & \textbf{0.0711} & \textbf{0.0603} & \textbf{0.0575} \\
    & $10^{-6}$ & 0.2679 & 0.2764 & 0.2865 & 0.3093 & 0.3169 & 0.3265 & 0.6596 & 0.6339 & 0.6347 & 0.0580 & 0.0795 & 0.0849 \\
    \midrule
    \multirow{7}{*}{\textbf{Grandis-CC}} 
    & $10^{0}$       & 0.0297 & 0.0303 & 0.0304 & 0.0183 & 0.0191 & 0.0190 & 0.0373 & 0.0390 & 0.0386 & 0.0085 & 0.0088 & 0.0317 \\
    & $10^{-1}$       & \textbf{0.0947} & \textbf{0.1020} & \textbf{0.0951} & \textbf{0.0328} & \textbf{0.0363} & \textbf{0.0325} & \textbf{0.4083} & \textbf{0.4101} & \textbf{0.4104} & \textbf{0.0019} & \textbf{0.0191} & \textbf{0.0318} \\
    & $10^{-2}$      & 0.4583 & 0.4726 & 0.4684 & 0.0433 & 0.0494 & 0.0481 & 0.4249 & 0.4251 & 0.4262 & 0.0072 & 0.0534 & 0.0839 \\
    & $10^{-3}$    & 0.2469 & 0.2511 & 0.2478 & 0.0289 & 0.0333 & 0.0310 & 0.4325 & 0.4325 & 0.4322 & 0.0025 & 0.0109 & 0.0436 \\
    & $10^{-4}$    & 0.1459 & 0.1568 & 0.1495 & 0.0356 & 0.0416 & 0.0364 & 0.4294 & 0.4309 & 0.4299 & 0.0019 & -0.0007 & 0.0401 \\
    & $10^{-5}$ & 0.1548 & 0.1697 & 0.1633 & 0.0742 & 0.0838 & 0.0780 & 0.4345 & 0.4389 & 0.4370 & 0.0092 & -0.0114 & 0.0138 \\
    & $10^{-6}$ & 0.3353 & 0.3497 & 0.3437 & 0.1724 & 0.1825 & 0.1737 & 0.4262 & 0.4285 & 0.4279 & 0.0032 & 0.0104 & 0.0370 \\
    \midrule
    \multirow{7}{*}{\textbf{Grandis-RC}} 
    & $10^{0}$       & 0.0865 & 0.0951 & 0.0989 & 0.0193 & 0.0227 & 0.0231 & 0.0352 & 0.0407 & 0.0415 & 0.0139 & -0.0302 & -0.0098 \\
    & $10^{-1}$       & 0.1685 & 0.1979 & 0.1968 & 0.0318 & 0.0412 & 0.0414 & 0.2593 & 0.2879 & 0.2882 & 0.0229 & -0.0058 & 0.0049 \\
    & $10^{-2}$      & 0.0340 & 0.0457 & 0.0454 & 0.0251 & 0.0340 & 0.0331 & 0.2560 & 0.3202 & 0.3186 & -0.0024 & 0.0100 & 0.0077 \\
    & $10^{-3}$     & 0.0294 & 0.0389 & 0.0388 & 0.0219 & 0.0275 & 0.0272 & 0.3174 & 0.3388 & 0.3423 & 0.0145 & -0.0166 & 0.0025 \\
    & $10^{-4}$    & 0.0382 & 0.0504 & 0.0490 & 0.0301 & 0.0375 & 0.0351 & 0.3263 & 0.3388 & 0.3407 & 0.0170 & -0.0232 & -0.0057 \\
    & $10^{-5}$ & 0.0760 & 0.0972 & 0.0963 & 0.0713 & 0.0825 & 0.0803 & 0.3414 & 0.3597 & 0.3611 & -0.0006 & 0.0028 & 0.0045 \\
    & $10^{-6}$ & \textbf{0.1595} & \textbf{0.1732} & \textbf{0.1742} & \textbf{0.2079} & \textbf{0.2194} & \textbf{0.2196} & \textbf{0.3377} & \textbf{0.3617} & \textbf{0.3644} & \textbf{0.0249} & \textbf{-0.0587} & \textbf{-0.0568} \\
    \bottomrule
  \end{tabular}%
  }
  
  \vspace{0.4cm}
  
  \resizebox{\textwidth}{!}{%
  \begin{tabular}{clcccccccccccc}
    \toprule
    \multicolumn{14}{c}{\textbf{Part II: PInteCNN Learned Physical Parameters}} \\
    \midrule
    \multirow{2}{*}{\textbf{Dataset}} & \multirow{2}{*}{\textbf{Steps}} & \multicolumn{3}{c}{\textbf{$\theta_{\alpha}$}} & \multicolumn{3}{c}{\textbf{$\theta_{\beta}$}} & \multicolumn{3}{c}{\textbf{$\theta_{\text{bed}}$}} & \multicolumn{3}{c}{\textbf{$\theta_{\text{s}}$}} \\
    \cmidrule(lr){3-5} \cmidrule(lr){6-8} \cmidrule(lr){9-11} \cmidrule(lr){12-14}
    & & Train & Val & Test & Train & Val & Test & Train & Val & Test & Train & Val & Test \\
    \midrule
    \multirow{3}{*}{\textbf{Poplar}} 
    & $0$ & 0.4332 & 0.4320 & 0.4322 & 0.3682 & 0.3699 & 0.3700 & 0.6078 & 0.6019 & 0.5973 & 0.0491 & 0.0612 & 0.0856 \\
    & $1$ & \textbf{0.3342} & \textbf{0.3381} & \textbf{0.3430} & \textbf{0.2336} & \textbf{0.2386} & \textbf{0.2432} & \textbf{0.6243} & \textbf{0.6216} & \textbf{0.6193} & \textbf{0.0336} & \textbf{0.0413} & \textbf{0.0478} \\
    & $2$ & 0.2656 & 0.2612 & 0.2640 & 0.1839 & 0.1784 & 0.1806 & 0.6509 & 0.6507 & 0.6492 & 0.0492 & 0.0421 & 0.0423 \\
    \midrule
    \multirow{3}{*}{\textbf{Grandis-CC}} 
    & $0$ & 0.4412 & 0.4442 & 0.4407 & 0.3601 & 0.3675 & 0.3649 & 0.4768 & 0.4849 & 0.4847 & 0.0137 & 0.0041 & 0.0362 \\
    & $1$ & \textbf{0.2575} & \textbf{0.2593} & \textbf{0.2604} & \textbf{0.1759} & \textbf{0.1799} & \textbf{0.1790} & \textbf{0.4516} & \textbf{0.4529} & \textbf{0.4525} & \textbf{0.0073} & \textbf{-0.0049} & \textbf{0.0316} \\
    & $2$ & 0.2860 & 0.2878 & 0.2885 & 0.1882 & 0.1907 & 0.1914 & 0.5586 & 0.5607 & 0.5585 & 0.0547 & 0.0670 & 0.0496 \\
    \midrule
    \multirow{3}{*}{\textbf{Grandis-RC}} 
    & $0$ & 0.3322 & 0.3401 & 0.3429 & 0.2651 & 0.2727 & 0.2753 & 0.3767 & 0.3843 & 0.3875 & 0.0410 & 0.0628 & 0.0499 \\
    & $1$ & \textbf{0.2347} & \textbf{0.2479} & \textbf{0.2498} & \textbf{0.1940} & \textbf{0.2016} & \textbf{0.2035} & \textbf{0.3797} & \textbf{0.3892} & \textbf{0.3907} & \textbf{0.0311} & \textbf{0.0332} & \textbf{0.0515} \\
    & $2$ & 0.2519 & 0.2541 & 0.2526 & 0.1709 & 0.1742 & 0.1727 & 0.3689 & 0.3690 & 0.3693 & 0.0175 & 0.0087 & 0.0342 \\
    \bottomrule
  \end{tabular}%
  }
\end{table}

\section{Sequential Data Integration Analysis}
\label{sec:data_integration}
\begin{table}[t]
\centering
\caption{
    Ablation study of the three thermal prediction approaches across the Poplar, Grandis-CC, and Grandis-RC datasets. 
    We compare the MAEs of Data-Driven, PICNN, and PInteCNN across sequential data integration steps. 
    Results include Best-First, Worst-First, Random, and Average (Main Path) strategies.
}
\label{tab:full_ablation_master}
\resizebox{\textwidth}{!}{%
\begin{tabular}{l cccc cccc cccc}
\toprule
& \multicolumn{4}{c}{\textbf{Data-Driven}} & \multicolumn{4}{c}{\textbf{PICNN}} & \multicolumn{4}{c}{\textbf{PInteCNN}} \\
\cmidrule(lr){2-5} \cmidrule(lr){6-9} \cmidrule(lr){10-13}
\textbf{Training Sequence} & \textbf{Best} & \textbf{Worst} & \textbf{Rand} & \textbf{Avg} & \textbf{Best} & \textbf{Worst} & \textbf{Rand} & \textbf{Avg} & \textbf{Best} & \textbf{Worst} & \textbf{Rand} & \textbf{Avg} \\
\midrule
\multicolumn{13}{c}{\textbf{Poplar}} \\
\midrule
Step 1 ($\mathcal{D}_1$) & 0.3405 & 0.3313 & 0.3411 & 0.3390 & 0.3389 & 0.3249 & 0.3345 & 0.3335 & 0.3363 & 0.3388 & 0.3388 & 0.3383 \\
Step 2 ($\mathcal{D}_{1..2}$) & 0.3177 & 0.3279 & 0.3227 & 0.3228 & 0.3187 & 0.3306 & 0.3217 & 0.3229 & 0.3229 & 0.3272 & 0.3249 & 0.3250 \\
Step 3 ($\mathcal{D}_{1..3}$) & 0.3154 & 0.3138 & 0.3129 & 0.3136 & 0.3175 & 0.3183 & 0.3158 & 0.3166 & 0.3221 & 0.3122 & 0.3171 & 0.3171 \\
Step 4 ($\mathcal{D}_{1..4}$) & 0.3116 & 0.3158 & 0.3129 & 0.3132 & 0.3124 & 0.3132 & 0.3140 & 0.3135 & 0.3077 & 0.3127 & 0.3166 & 0.3140 \\
Step 5 ($\mathcal{D}_{1..5}$) & 0.3105 & 0.3060 & 0.3150 & 0.3123 & 0.3160 & 0.3110 & 0.3113 & 0.3122 & 0.3116 & 0.3114 & 0.3131 & 0.3124 \\
Step 6 ($\mathcal{D}_{1..6}$) & 0.3116 & 0.3116 & 0.3116 & 0.3116 & 0.3088 & 0.3088 & 0.3088 & 0.3088 & 0.3096 & 0.3096 & 0.3096 & 0.3096 \\
\midrule
\multicolumn{13}{c}{\textbf{Grandis-CC}} \\
\midrule
Step 1 ($\mathcal{D}_1$) & 0.2481 & 0.2435 & 0.2581 & 0.2532 & 0.2584 & 0.2567 & 0.2702 & 0.2651 & 0.2835 & 0.2723 & 0.2853 & 0.2823 \\
Step 2 ($\mathcal{D}_{1..2}$) & 0.2209 & 0.2259 & 0.2411 & 0.2340 & 0.2322 & 0.2314 & 0.2538 & 0.2450 & 0.2451 & 0.2848 & 0.2632 & 0.2639 \\
Step 3 ($\mathcal{D}_{1..3}$) & 0.2198 & 0.2304 & 0.2278 & 0.2267 & 0.2313 & 0.2264 & 0.2372 & 0.2339 & 0.2302 & 0.2278 & 0.2339 & 0.2319 \\
Step 4 ($\mathcal{D}_{1..4}$) & 0.2096 & 0.2152 & 0.2145 & 0.2136 & 0.2142 & 0.2261 & 0.2192 & 0.2196 & 0.2227 & 0.2267 & 0.2254 & 0.2251 \\
Step 5 ($\mathcal{D}_{1..5}$) & 0.2014 & 0.2008 & 0.2117 & 0.2075 & 0.2098 & 0.2081 & 0.2156 & 0.2129 & 0.2155 & 0.2165 & 0.2186 & 0.2176 \\
Step 6 ($\mathcal{D}_{1..6}$) & 0.2051 & 0.2051 & 0.2051 & 0.2051 & 0.2043 & 0.2043 & 0.2043 & 0.2043 & 0.2101 & 0.2101 & 0.2101 & 0.2101 \\
\midrule
\multicolumn{13}{c}{\textbf{Grandis-RC}} \\
\midrule
Step 1 ($\mathcal{D}_1$) & 0.4734 & 0.4343 & 0.4242 & 0.4361 & 0.5832 & 0.4095 & 0.4106 & 0.4449 & 0.4812 & 0.4484 & 0.4447 & 0.4528 \\
Step 2 ($\mathcal{D}_{1..2}$) & 0.4258 & 0.4350 & 0.3950 & 0.4092 & 0.4250 & 0.3991 & 0.3793 & 0.3924 & 0.4504 & 0.4489 & 0.4260 & 0.4355 \\
Step 3 ($\mathcal{D}_{1..3}$) & 0.4070 & 0.3738 & 0.3737 & 0.3804 & 0.4034 & 0.3742 & 0.3831 & 0.3854 & 0.4501 & 0.4250 & 0.4363 & 0.4368 \\
Step 4 ($\mathcal{D}_{1..4}$) & 0.3709 & 0.3436 & 0.3660 & 0.3625 & 0.3641 & 0.3562 & 0.3601 & 0.3601 & 0.3885 & 0.4432 & 0.3833 & 0.3963 \\
Step 5 ($\mathcal{D}_{1..5}$) & 0.3634 & 0.3512 & 0.3551 & 0.3560 & 0.3515 & 0.3692 & 0.3568 & 0.3582 & 0.4368 & 0.3903 & 0.4206 & 0.4178 \\
Step 6 ($\mathcal{D}_{1..6}$) & 0.3578 & 0.3578 & 0.3578 & 0.3578 & 0.3456 & 0.3456 & 0.3456 & 0.3456 & 0.3497 & 0.3497 & 0.3497 & 0.3497 \\
\bottomrule
\end{tabular}%
}
\end{table}
\begin{figure}[htbp]
  \centering
  \includegraphics[width=\textwidth]{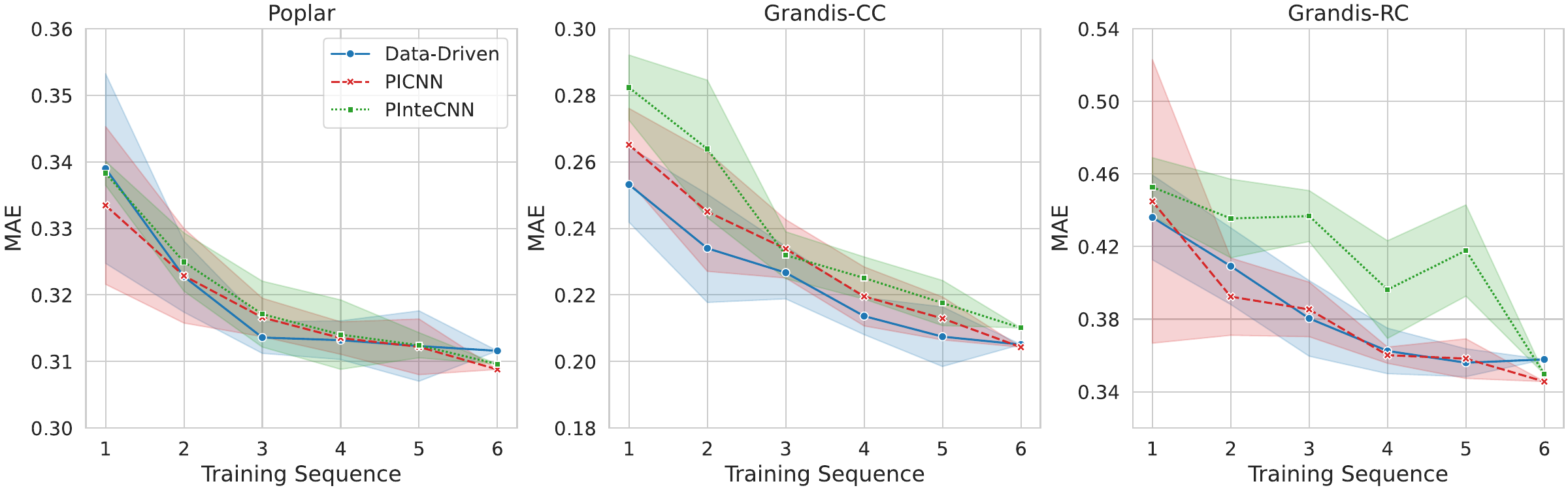} 
  \caption{
    Comparison of the Main Path for Data-Driven, PICNN, and PInteCNN across sequential data integration steps. 
    We evaluate on the Poplar, Grandis-CC, and Grandis-RC datasets using MAE.
    The trend lines represent the average MAE across all integration strategies (Best-First, Worst-First, and Random sequences), illustrating the overall learning trajectory of each model as more data is introduced. 
    The shaded regions denote the corresponding standard deviation, indicating the performance variance introduced by the specific choice of the integration path at each step.
  }
  \label{fig:main_path_comparison}
\end{figure}

\paragraph{Experimental Configurations.}
To empirically evaluate model robustness, we introduce diverse strategies for data integration. 
By fixing the optimal validation and test sets, $\mathcal{D}_{\text{val}}$ and $\mathcal{D}_{\text{test}}$, as the target distribution, we iteratively accumulate the six training sub-partitions into a cumulative training set, $\mathcal{D}_{\text{train}}$, via distinct topological pathways:

\begin{itemize}
    \item \textbf{Best-First Integration.} 
    A process that minimizes the cumulative KS distance relative to the target distribution, testing optimal distribution alignment scenarios.
    \item \textbf{Worst-First Integration.} 
    A process that maximizes the cumulative KS distance at each step, intentionally inducing maximum distributional discrepancy to test the extreme limits of model generalization.
    \item \textbf{Random Integration.}
    A stochastic baseline representing an unguided data accumulation process, executed across multiple random seeds to establish statistical norms.
\end{itemize}

To synthesize these distinct pathways into a singular, definitive metric, we establish the \textbf{Main Path (Expected Performance)}. 
Calculated as the average across the Best-First, Worst-First, and Random strategies (Table~\ref{tab:full_ablation_master} and Figure~\ref{fig:main_path_comparison}), the Main Path isolates the fundamental learning trajectory of each model as a function of data volume, smoothing out the variance introduced by specific sampling strategies.
The evaluation investigates how the models navigate data scarcity and scale by evaluating the interplay between three distinct factors: distribution alignment, intra-species diversity, and physical regularization. 
\begin{itemize}
    \item \textit{Distribution Alignment (Best-First)} tests whether strict statistical matching (minimizing KS distance) to the target set is required to stabilize the data-driven approach when training volumes are critically low.
    \item \textit{Intra-species diversity (Worst-First)} tests whether early exposure to physical extrema and boundary cases (maximizing KS distance) induces catastrophic domain shift.
    \item \textit{Physical Regularization} evaluates how physics-informed mechanisms alter the model's reliance on the previous two factors, directly impacting the variance (shaded regions in Figure~\ref{fig:main_path_comparison}) and expected performance.
\end{itemize}
This multi-axis evaluation isolates the precise conditions under which purely statistical learning fails, validating the necessity of physics-informed constraints when data is simultaneously limited, highly diverse, and structurally complex.

\paragraph{Taxonomy of Physics-Informed Robustness Across Dataset Topologies.}
The sequential ablation experiments analyze the interaction between the integration strategies and the thermal prediction approaches across the three wood datasets, characterizing the robustness and expected performance of the physics-informed models (Table~\ref{tab:full_ablation_master}).
\paragraph{Poplar.} 
On the Poplar dataset, the expected performance of the data-driven approach plateaus prematurely. While it successfully lowers the MAE during the initial integration phases, its learning stagnates by Step 3, ultimately flatlining near an MAE of $0.3116$. In contrast, the physics-informed models impose strict structural regularization that smooths the complex loss landscape. Starting at an average MAE of $0.3335$ in Step 1, PICNN enforces a predictable, monotonic descent, pushing past the statistical plateau to achieve the global minimum of $0.3088$ by Step 6, while PInteCNN remains competitive at $0.3096$. \\
\textbf{Conclusion.} 
When underlying dataset complexity causes the data-driven approach to prematurely plateau, physical constraints serve as a powerful regularizer, ensuring continuous, monotonic convergence and a lower asymptotic error limit.

\paragraph{Grandis-CC.} 
Under extreme data scarcity (Step 1), the data-driven approach performs well, achieving the lowest average MAE ($0.2532$) and outperforming both PICNN ($0.2651$) and PInteCNN ($0.2823$). However, as data volume scales, the data-driven approach experiences diminishing returns, and its rate of improvement slows. In contrast, PICNN maintains a steeper, more sustained descent trajectory. By Step 6, this scaling allows PICNN to marginally cross under the data-driven approach to achieve the lowest final MAE ($0.2043$ vs. $0.2051$). Interestingly, PInteCNN struggles to reconcile its integration mechanisms with this specific topology, lagging behind both at $0.2101$. \\
\textbf{Conclusion.} 
While the data-driven approach can rapidly adapt to limited data distributions, physical regularization (PICNN) sustains a more consistent learning trajectory over time, ultimately overtaking the statistical baseline as the broader dataset distribution is exposed.

\paragraph{Grandis-RC.}
Grandis-RC presents high structural diversity, leading to volatile early-stage behavior. At Step 1, PICNN exhibits substantial variance depending on the sampling strategy (ranging from an expected MAE of $0.5832$ under Best-First to $0.4095$ under Worst-First), allowing the data-driven approach to claim the best initial performance ($0.4361$). Yet, as data accumulates, a distinct inversion occurs. Both PICNN and PInteCNN exhibit robust learning trajectories, whereas the data-driven approach hits a performance plateau, stagnating near $0.3578$ from Step 4 onward. By Step 6, PICNN ($0.3456$) and PInteCNN ($0.3497$) decisively break through the performance ceiling of the data-driven approach.\\
\textbf{Conclusion.} 
Under high variance and complex structural diversity, physics-informed models initially struggle to reconcile physical laws with highly sparse data. However, once sufficient data is integrated, physical embedding provides a distinct advantage, allowing the models to map complex boundaries that statistical learning struggles to resolve effectively.

\end{document}